\documentclass{article}

\usepackage{iclr2026_conference,times}
\usepackage{latexsym}
\usepackage[T1]{fontenc}
\usepackage[utf8]{inputenc}
\usepackage{microtype}
\usepackage{hyperref}
\usepackage{url}
\hypersetup{
  colorlinks=true,
  linkcolor=black,
  citecolor=blue,
  urlcolor=blue
}
\usepackage{graphicx}
\usepackage{dcolumn}
\usepackage{booktabs}
\usepackage{amsfonts}
\usepackage{nicefrac}
\usepackage{enumitem}

\newcommand{\TabNumDelta}[1]{\rlap{\kern0.15em{\scriptsize$_{#1}$}}}
\newcommand{\VTCn}[1]{\multicolumn{1}{D{.}{.}{2.2}}{#1}}
\newcommand{\VTCDelta}[1]{\rlap{\kern0.12em{\scriptsize$_{#1}$}}}
\usepackage{amsmath}
\usepackage{bm}
\usepackage{multirow}
\usepackage{algorithm}
\usepackage{algpseudocode}
\usepackage{placeins}
\usepackage{float}
\usepackage{caption}
\usepackage{amsthm}
\theoremstyle{definition}
\newtheorem{definition}{Definition}
\theoremstyle{plain}

\newtheorem{proposition}{Proposition}

\usepackage{tcolorbox}
\usepackage{amssymb}
\usepackage{mathrsfs}

\newcommand{\logotitle}{\textbf{LedgerMind}}
\newcommand{\logofull}{Provenance-Constrained Multimodal Agentic Reasoning with a Structured Evidence Ledger}
\definecolor{baselineBlue}{RGB}{150,190,220}
\definecolor{ledgerMindPink}{RGB}{255,150,145}

\title{LedgerMind: A Structured Evidence Runtime for Auditable Multimodal Agent Trajectories}

\author{
Enjun Du$^{1,2,*}$ \quad
Hange Zhou$^{1,*}$ \quad
Chenxu Du$^{1}$ \quad
Siyi Liu$^{1}$ \\
\bfseries
Zirong Chen$^{1,3}$ \quad
Ziyu Zheng$^{4}$ \quad
Yongqi Zhang$^{1,\dagger}$ \\[3pt]
\normalfont
$^{1}$The Hong Kong University of Science and Technology (Guangzhou) \\
$^{2}$The University of Hong Kong \quad
$^{3}$Tsinghua University \quad
$^{4}$University of Sussex \\
$^{*}$Equal contribution. \quad
$^{\dagger}$Corresponding author. \\
\texttt{enjundu.cs@gmail.com}
}

\iclrfinalcopy
\begin{document}

\maketitle
\fancyhead{}
\renewcommand{\headrulewidth}{0pt}

\begin{abstract}

Multimodal agents for visual question answering increasingly
operate as multi-step trajectories that interleave perception,
retrieval, and reasoning, yet evaluation still largely reduces
to final-answer accuracy. This aggregate signal cannot tell
whether a correct answer was reached through grounded evidence,
language priors, or accidental error cancellation. We propose
to treat a multimodal agent trajectory as a
provenance-constrained state machine: tool outputs are
normalized into a Structured Evidence Ledger that serves as
the trajectory state, downstream
reasoning and decision claims may cite only active ledger
entries, grounding is checked at the entity and numeric level,
and repair is realized as typed state transitions that cannot
introduce content without tool-produced provenance. We
instantiate this design as \logotitle{} (\logofull{}), augmented by a Three-Layer Grounding
Protocol, an Adaptive Dual-Path Dispatcher that matches
reasoning depth to question complexity, and an Event-Triggered
Verification-and-Repair engine with a formal provenance
non-amplification guarantee. We use \logotitle{} to target four
recurring failure patterns that final-answer accuracy tends to
obscure: unsupported intermediate reasoning, citation-backed
entity hallucination (Phantom Grounding), over-reasoning on
simple queries, and repair-time amplification. Experiments across multiple
multimodal reasoning benchmarks and backbone MLLMs show that
\logotitle{} improves both answer accuracy and
trajectory-level faithfulness.
\end{abstract}

\section{Introduction}
\label{sec:intro}

Recent multimodal large language models (MLLMs) have enabled
visual question answering systems that interleave perception,
retrieval, and language-based
reasoning~\cite{openai2023gpt4,gemini2024,liu2023llava,bai2023qwenvl,
chen2024internvl}. As these systems become
agentic~\cite{yao2023react,schick2023toolformer,lu2024chameleon},
their outputs are no longer single predictions but multi-step
trajectories containing observations, intermediate claims,
tool calls, and final decisions. Yet final-answer accuracy
remains the primary aggregate signal for such systems, and
it cannot tell whether a correct answer was obtained through
grounded evidence, language priors, or accidental error
cancellation.

\paragraph{From accuracy to trajectory faithfulness.}
Multimodal agentic reasoning requires trajectory-level
faithfulness: intermediate reasoning claims should be
auditable against the evidence that supports them. This is
difficult because most agent frameworks store the trajectory
as a free-form text
buffer~\cite{yao2023react,shinn2023reflexion} in which tool
outputs, model paraphrases, inferred facts, and repaired
claims are mixed together, so a claim can appear plausible,
or even cite an evidence identifier, while introducing
entities or numerical values absent from the cited evidence.
Natural-language rationales frequently fail to reflect the
process that actually produced the
answer~\cite{turpin2023unfaithful,lanham2023measuring,lyu2023faithful,paul2024makingreasoningmatter},
and analogous citation-content mismatches are documented in
retrieval-augmented
generation~\cite{liu2023verifiability,gao2023alce,muller2024correctness}.

Existing remedies only partially address this issue.
Requiring models to cite evidence does not suffice, because
citation structure can mask conclusion-level hallucination;
we call this failure Phantom Grounding, a form of spurious
grounding in which a claim cites valid evidence IDs yet
introduces entities or numerical values absent from the
cited source. Free-form
self-refinement~\cite{madaan2023selfrefine,shinn2023reflexion,wang2023selfconsistency,dhuliawala2024cove,yin2024woodpecker,huang2024opera}
is also insufficient, since LLMs largely cannot self-correct
without external
feedback~\cite{huang2024selfcorrect,stechly2024selfverification,xu2024pride}
and repair itself can introduce new unsupported claims while
fixing old ones (repair-time amplification). Finally, deeper
reasoning is not uniformly
beneficial~\cite{sui2025overthinking,chen2024overthinking,hassid2025dontoverthink,han2025empiricalreasoninglength}:
unnecessary multi-step pipelines can overwrite simple visual
or factual answers with noisy intermediate inferences.

We propose \logotitle{} (\logofull{}), a
provenance-constrained framework for faithful multimodal
agentic reasoning. Rather than storing the trajectory as an
unstructured text buffer, \logotitle{} normalizes each tool
output into an entry of a Structured Evidence Ledger
carrying source, type, confidence, lifecycle status, and
dependencies; downstream reasoning and decision claims may
cite only active ledger entries, making provenance a
structural constraint rather than a prompting preference. A
Three-Layer Grounding Protocol then verifies, beyond
structural coverage, whether cited evidence actually
contains the entities and numerical values used in the
claim. An Adaptive Dual-Path Dispatcher avoids unnecessary
deep reasoning on simple or knowledge-oriented queries, and
an Event-Triggered Verification-and-Repair engine modifies
the ledger or actions only through predefined typed
operators, yielding a provenance-level guarantee that repair
cannot introduce ledger entries without tool-produced
provenance.

\begin{figure}[h]
    \centering
    \includegraphics[width=\textwidth]{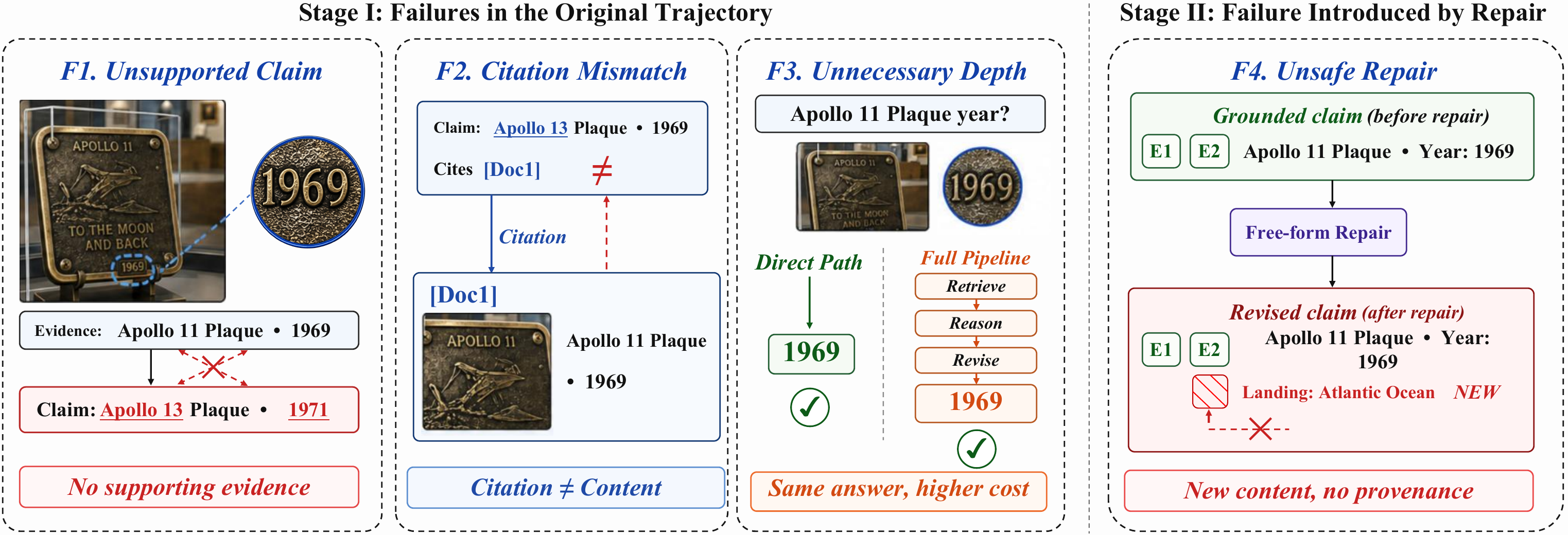}
    \caption{Representative failure patterns in multimodal
    agentic reasoning that final-answer accuracy tends to
    obscure.}
    \label{fig:intro_case}
\end{figure}

This design targets four recurring failure patterns that
aggregate accuracy tends to obscure
(Figure~\ref{fig:intro_case}): (F1) the Trajectory
Faithfulness Gap, where correct answers coexist with
unsupported intermediate claims; (F2) Phantom Grounding,
where citation structure passes but conclusion-level
entities are fabricated; (F3) the Over-Reasoning Paradox,
where additional reasoning steps overwrite an initially
correct answer on knowledge-intensive queries; and (F4)
Repair-Time Amplification, where free-form reflection
injects new unsupported claims while fixing old ones. We
treat F1--F4 as recurring and actionable patterns rather
than an exhaustive taxonomy; they serve as a diagnostic
lens in our trajectory-level evaluation.

Our contributions are as follows:
\begin{enumerate}[leftmargin=*]
    \item We introduce the Structured Evidence Ledger as
    the central state abstraction for multimodal agentic
    reasoning. Instead of treating trajectories as
    free-form text histories, the ledger stores
    tool-produced evidence with source, type, confidence,
    lifecycle status, and dependencies, so that downstream
    claims can be checked against active provenance rather
    than against prompt-level citation instructions alone.

    \item We formulate trajectory-level faithfulness as a
    provenance-constrained reasoning problem and identify
    four recurring failure patterns that final-answer
    accuracy obscures, including Phantom Grounding, which
    we formalize at the level of multimodal agent
    trajectories.

    \item We instantiate these ideas in \logotitle{}, a
    ledger-centered framework that binds claims to
    tool-produced evidence, verifies conclusion-level
    entity and numeric consistency, adapts reasoning depth
    to question complexity, and performs event-triggered
    typed repair with a provenance non-amplification
    guarantee (Proposition~\ref{prop:nonamp}).

    \item We evaluate \logotitle{} with trajectory-level
    faithfulness diagnostics on five answer-level multimodal
    benchmarks, one chain-level multimodal search benchmark,
    and Hard-200, a stress-test set of 200 complex
    image-grounded queries
    (Appendix~\ref{app:hard200}), showing consistent gains
    in both answer accuracy and grounded reasoning behavior
    across multiple backbone MLLMs.
\end{enumerate}

\section{Related Works}

\vspace{2mm}

\subsection{Multimodal Agentic Reasoning and Adaptive Inference}
\label{sec:related_agent}

\paragraph{Multimodal agents and tool-augmented reasoning.}
Recent MLLMs~\cite{openai2023gpt4,gemini2024,bai2023qwenvl,liu2023llava,chen2024internvl}
have extended text-only LLMs to jointly reason over images
and language. To tackle tasks beyond single-pass perception,
a line of work casts MLLMs as agents that interleave
reasoning with external tools:
ReAct~\cite{yao2023react} alternates thought and action
traces; Toolformer~\cite{schick2023toolformer} teaches
models to invoke APIs in a self-supervised manner;
ViperGPT~\cite{suris2023vipergpt} and
Chameleon~\cite{lu2024chameleon} compose vision-and-language
modules via program synthesis; and retrieval
augmentation~\cite{lewis2020rag} further grounds generation
in external knowledge. Recent systems broaden this design
space through history-aware routing over large tool
ecosystems (\mbox{ToolACE-MCP})~\cite{yao2025toolacemcp},
capability-level multimodal orchestration
(\mbox{Octopus})~\cite{guo2025octopusagenticmultimodalreasoning},
and specialized multi-agent synthesis
(\mbox{GraphMaster})~\cite{du2025graphmaster}. Long-horizon
agents further require persistent state across evolving
interactions, as studied by the RealMem
benchmark~\cite{bian2025realmem}. These frameworks commonly
treat the reasoning trace as an unstructured text buffer,
concatenating intermediate claims, tool outputs, and inferred
facts without explicit provenance, confidence, or lifecycle
metadata. Our framework departs from this practice by turning
the trajectory into a provenance-constrained ledger.

Structured execution records and provenance tracking have a
long systems tradition. Database and workflow provenance
capture the origin and causal history of derived
data~\cite{buneman2001whywhere,moreau2011opm}, while
distributed tracing records cross-component execution for
diagnosis~\cite{sigelman2010dapper}. We therefore do not
claim that structured tracing or dependency metadata is new
in isolation; our focus is their operational use in a shared
evidence state for multimodal agent grounding, adaptive
execution, and constrained repair.

Chain-of-thought (CoT) prompting~\cite{wei2022cot} and
extensions such as Self-Consistency~\cite{wang2023selfconsistency}
and Tree-of-Thoughts~\cite{yao2023tot} are standard
techniques for eliciting multi-step reasoning. Yet deeper
reasoning is not always better: recent analyses reveal an
overthinking phenomenon where unnecessarily long chains
waste computation and may even degrade
accuracy~\cite{sui2025overthinking,chen2024overthinking,hassid2025dontoverthink,han2025empiricalreasoninglength}.
Most adaptive inference methods focus on compressing or
truncating reasoning length within a single mode and do
not distinguish question types benefiting from deep
multi-step reasoning from those better served by shallow
retrieval-augmented answering, especially in multimodal
settings. This motivates a dispatcher-based control policy
instead of a uniform pipeline.

\subsection{Hallucination Mitigation and Trajectory-Level Faithfulness}
\label{sec:related_hallu}

\paragraph{Self-correction and verification.}
Prior work improves reasoning reliability through post-hoc
verification or repair:
Self-Refine~\cite{madaan2023selfrefine} iteratively critiques
and rewrites outputs, Reflexion~\cite{shinn2023reflexion}
learns from verbal feedback across trials, and
Chain-of-Verification~\cite{dhuliawala2024cove} plans and
answers fact-check questions. SE-Agent instead optimizes
multi-step agent trajectories through revision,
recombination, and
refinement~\cite{lin2025seagentselfevolutiontrajectoryoptimization}.
However, Huang~et~al.~\cite{huang2024selfcorrect} and
subsequent
analyses~\cite{stechly2024selfverification,xu2024pride} show
that LLMs cannot reliably self-correct without external
feedback: free-form reflection is unconstrained and may
itself introduce new unsupported claims, so repair based
purely on natural-language self-reflection carries an
inherent risk of error amplification. Our framework avoids
this by restricting repair to a finite set of typed
state-transition operators.

\paragraph{Hallucination mitigation and trajectory-level
evaluation.}
Hallucinations are a pervasive failure mode of
MLLMs~\cite{rawte2023hallucination,bai2024hallucination},
addressed by training-time methods that alter data or
objectives and by inference-time methods that correct
errors without retraining:
Woodpecker~\cite{yin2024woodpecker} extracts concepts and
verifies them with external detectors, and
OPERA~\cite{huang2024opera} introduces an over-trust
penalty during decoding. Most target caption-level object
hallucination on static MLLM outputs and are evaluated by
final-answer accuracy alone. A related line on
retrieval-augmented generation distinguishes citation
correctness from citation
faithfulness~\cite{liu2023verifiability,gao2023alce,muller2024correctness};
our Phantom Grounding notion instantiates this distinction
at the level of multimodal agent trajectories. Recent
benchmarks for multimodal chain-of-thought
quality~\cite{wang2025mmecot} begin to assess reasoning
robustness beyond accuracy, yet the faithfulness of
intermediate reasoning steps in tool-using agentic
pipelines remains comparatively underexplored. Claim-level
decomposition has also been used as a training signal: CuRe
constructs category-aware atomic claims for dense
video-caption
rewards~\cite{gao2026claimlevelrubricrewardsvideo}. Adjacent
video settings combine explicit reasoning with
segmentation~\cite{xu2025videosegr1reasoningvideoobjectsegmentation},
while continual segmentation separates class recognition
from mask generation to retain earlier
knowledge~\cite{guo2025decouplingcontinualsemanticsegmentation}.

\vspace{-2mm}
\begin{figure*}[htp]
    \centering
    \includegraphics[width=0.99\textwidth]{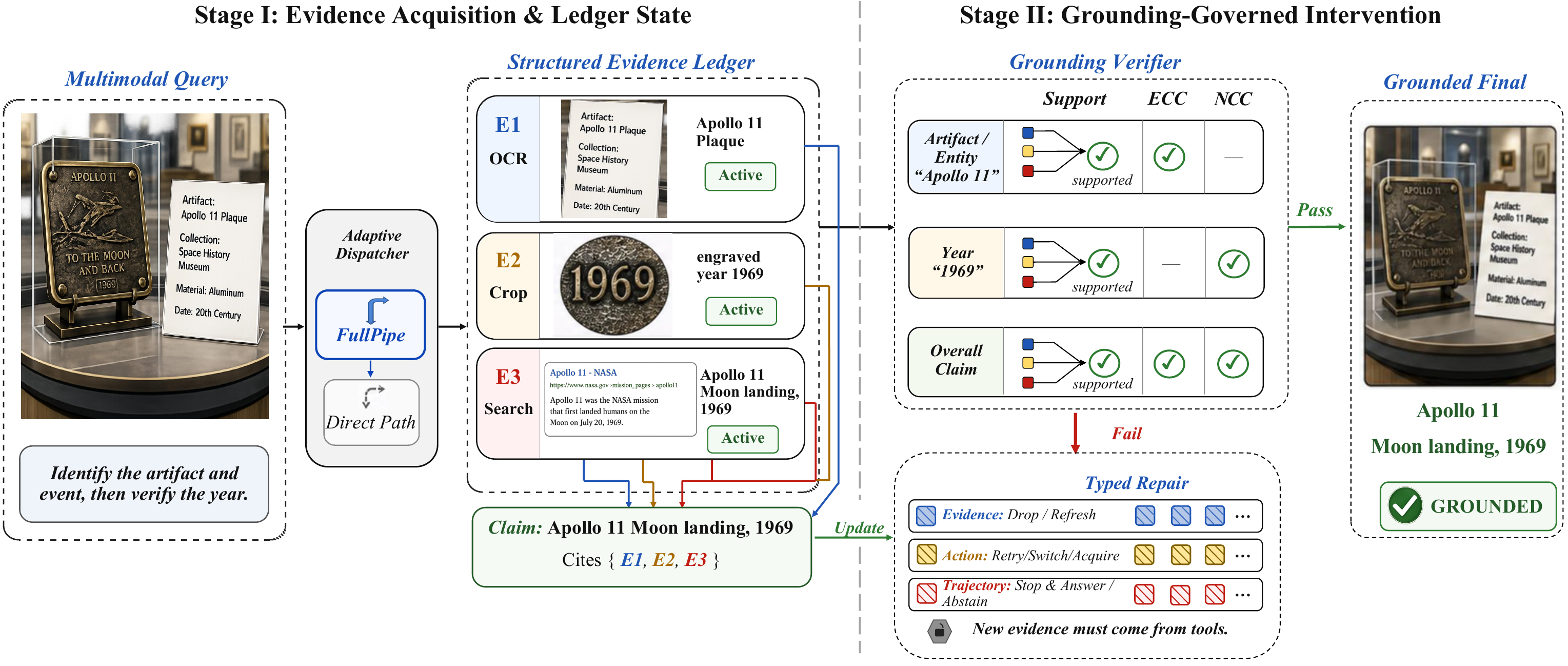}
    \vskip -4pt
    \caption{The \logotitle{} framework. The Structured
    Evidence Ledger is the central trajectory state: tool
    outputs are normalized into ledger entries, reasoning and
    decision claims cite active entries, grounding checks
    claim--evidence consistency, the dispatcher controls
    evidence-acquisition depth, and typed repair modifies only
    ledger state or tool actions.}
    \vskip -8pt
    \label{fig:framework}
\end{figure*}
\vspace{-2mm}

\section{Method}
\label{sec:method}

Figure~\ref{fig:framework} summarizes the two-stage
\logotitle{} runtime. Stage I routes the multimodal query
through the adaptive dispatcher and records tool-produced
perception and retrieval evidence as active ledger entries.
Stage II checks the resulting claim through support coverage,
ECC, and NCC. A passing claim yields the grounded final
answer; a failing claim can invoke only typed evidence-,
action-, or trajectory-level repair before the ledger is
updated.

\logotitle{} treats multimodal agentic reasoning as a
provenance-constrained state machine. Its central state is a
Structured Evidence Ledger: every tool return is normalized
into ledger evidence, every downstream reasoning or decision
claim is checked against active ledger entries, and every
repair is a typed transition on the same state. Thus the
grounding protocol, the dispatcher, and the repair engine are
not separate patches for separate failures; they are
ledger-facing operations that enforce the same provenance
contract at different points of the trajectory. Complete
schemas, trigger rules, classifier details, threshold choices,
prompts, and full pseudocode are given in
Appendices~\ref{app:algorithm}--\ref{app:prompts}.

\subsection{Structured Evidence Ledger as Trajectory State}
\label{sec:prelim}
\label{sec:ledger}

Given a question $q$ and multimodal context
$\mathcal{I}=\{I_1,\dots,I_K\}$, an agent produces a
trajectory $\tau=(s_1,\dots,s_T)$. Each step
$s_t=(a_t,o_t,\mathcal{C}_t)$ contains an action $a_t$, a tool
observation $o_t$, and a set of claims $\mathcal{C}_t$.
Claims are typed as Observation Claims (OC), which directly
record tool outputs; State Claims (SC), which aggregate or
infer over evidence; and Decision Claims (DC), which determine
the final answer $\hat{y}$. We write
$\mathcal{C}_\tau=\bigcup_t\mathcal{C}_t$ with subsets
$\mathcal{C}_\tau^{\textsc{OC}}$,
$\mathcal{C}_\tau^{\textsc{SC}}$, and
$\mathcal{C}_\tau^{\textsc{DC}}$.

Prior agent frameworks~\citep{yao2023react,shinn2023reflexion}
typically store trajectories as free-form transcripts, where
tool observations, model paraphrases, inferred facts, and
later revisions are concatenated in one text buffer. The
ledger $\mathcal{L}_t$ replaces this transcript with the
canonical runtime state. Each entry $e\in\mathcal{L}_t$
minimally stores its source tool, epistemic type
$\kappa(e)\in\{\textsc{Perception},\textsc{Retrieval},
\textsc{Derivation}\}$, rule-normalized fact $f_e$, confidence
$\sigma_e\in[0,1]$, lifecycle status
$\omega_e\in\{\textsc{Active},\textsc{Stale},\textsc{Conflicted},
\textsc{Dropped}\}$, and dependency links to the claims that
cite it. The fact $f_e$ is obtained by a deterministic
tool-specific mapping $\mathcal{M}:\mathcal{O}\rightarrow
\Sigma^*$, such as OCR text, a search snippet, or crop
metadata, rather than by LLM paraphrasing. The full 11-field
schema, time-to-live policy, lifecycle operations, and
dependency graph are specified in Appendix~\ref{app:ledger_full}.

The ledger enforces four invariants. (I1) \emph{Evidence
origination}: every evidence entry is either a direct tool
return or the image of a tool return under $\mathcal{M}$.
(I2) \emph{Citation validity}: every SC or DC must cite active
ledger entries before it can influence later reasoning. (I3)
\emph{Claim--evidence containment}: citation is necessary but
not sufficient; conclusion-level entities and numerical values
must also be licensed by the cited evidence pool. (I4)
\emph{Repair locality}: repair may change entry status or
invoke tools, but cannot append unsupported free-form content
to the ledger.

Let
$\mathcal{L}^{\textsc{Active}}_t=\{e\in\mathcal{L}_t:\omega_e=\textsc{Active}\}$.
For a claim $c$, its support set is resolved as
\begin{equation}
    S(c)=\{e\in\mathcal{L}^{\textsc{Active}}_t: e\text{ is cited by }c\}.
    \label{eq:support_set}
\end{equation}

A non-judgment claim with $S(c)=\emptyset$ is dropped, while a
judgment claim with $S(c)=\emptyset$ is demoted below the
verification threshold before it can affect the decision. In
this sense, provenance is a structural precondition for
reasoning rather than a prompt-level preference.

These invariants instantiate the runtime objective
\begin{equation}
    \max_{\pi}\;
    \mathbb{E}_{\tau\sim p_\pi}\!\left[
    S(\tau)-\sum_{r\in\mathcal{R}_\tau}
    \lambda_{\mathrm{type}(r)}\,\mathrm{cost}(r)\right]
    \quad\mathrm{s.t.}\quad
    \mathrm{UCR}_{\mathrm{reason}}(\tau)\le\epsilon .
    \label{eq:obj}
\end{equation}

Here, $S(\tau)\in\{0,1\}$ is final-answer success,
$\mathcal{R}_\tau$ is the set of repair operations, and
$\lambda_k\ge0$ is the cost weight for repair type
$k\in\{\textsc{Evidence},\textsc{Action},\textsc{Trajectory}\}$.
Eq.~\ref{eq:obj} formalizes the runtime trade-off controlled
by evidence acquisition and repair; \logotitle{} is otherwise
training-free and uses frozen backbone MLLMs.

\subsection{Ledger-Governed Reasoning and Control}
\label{sec:ledger_ops}

At runtime, \logotitle{} executes a ledger-governed control
loop: acquire evidence, validate claims against the active
ledger, select the necessary depth of evidence acquisition,
and repair only when a ledger-level violation is detected. We
describe these operations at the interface level here and
defer implementation rules to the appendices.

\paragraph{Grounding as claim--evidence containment.}
\label{sec:grounding}
For containment checks, DERIVATION entries are transparent rather than
self-supporting: when a claim cites a DERIVATION entry, we recursively resolve
that entry to its PERCEPTION and RETRIEVAL leaf dependencies. The cited pool is
therefore built only from leaf evidence,
\[
\mathcal{P}_c=\{f_e:e\in S_{\mathrm{leaf}}(c),\;
\kappa(e)\in\{\textsc{Perception},\textsc{Retrieval}\}\}.
\] First, Support Coverage
requires lexical overlap with active evidence,
\begin{equation}
    \rho(c)=
    \frac{|\mathrm{Tok}(c)\cap\bigcup_{e\in S(c)}\mathrm{Tok}(f_e)|}
         {|\mathrm{Tok}(c)|},
    \label{eq:rho}
\end{equation}
where $\mathrm{Tok}(\cdot)$ removes stopwords and extracts
content tokens. Second, Entity Consistency Check (ECC)
requires all conclusion-level non-numeric entities to appear
in the cited pool up to aliases:
\begin{equation}
    \mathrm{ECC}(c)=\mathbf{1}\!
    \left[\mathrm{Ent}(c)\subseteq
    \mathrm{Ent}(\mathcal{P}_c)\cup\mathrm{Alias}(\mathcal{P}_c)\right].
    \label{eq:ecc}
\end{equation}

Third, Numeric Coherence Check (NCC) requires each numerical
value in $c$ to match a cited value under a type-aware
tolerance $\Delta_u$:
\begin{equation}
    \mathrm{NCC}(c)=\mathbf{1}\!\left[
    \forall v_c\in\mathrm{Num}(c),\;
    \exists v_e\in\mathrm{Num}(\mathcal{P}_c):\
    |v_c-v_e|\le\Delta_{u(v_c)}(v_e)\right].
    \label{eq:ncc}
\end{equation}

For years, counts, dates, option labels, and identifiers,
$\Delta_u$ is exact; for continuous visual readings it allows
a small relative tolerance. Failure of ECC or NCC marks a
conclusion--evidence mismatch and lowers claim confidence below
the verification threshold, triggering hypothesis verification
when budget permits. Extraction rules, alias handling,
confidence demotion values, and threshold sensitivity are given
in Appendices~\ref{app:thresholds} and~\ref{app:algorithm}.

\paragraph{Adaptive evidence-depth control.}
\label{sec:dispatcher}
A ledger constraint does not imply that every query should use
the deepest pipeline. Unnecessary tool calls may add irrelevant
evidence and give the model more opportunities to overwrite a
simple answer. \logotitle{} therefore chooses between two
ledger-compatible paths before expensive evidence acquisition:
\begin{equation}
    \hat{y}=
    \begin{cases}
        \textsc{Direct}(q,\mathcal{I}) & \phi(q)=\texttt{simple},\\[2pt]
        \textsc{FullPipe}(q,\mathcal{I}) & \phi(q)=\texttt{complex},
    \end{cases}
    \label{eq:dispatch}
\end{equation}
where $\phi(q)$ is a deterministic complexity classifier. The
\textsc{Direct} path uses a small evidence budget for
single-step visual or shallow factual queries; the
\textsc{FullPipe} path performs task planning, OC extraction,
optional crop-zoom and Dual-Read Verification, retrieval,
grounded
$[\mathrm{E}]\!\rightarrow\![\mathrm{I}]\!\rightarrow\![\mathrm{J}]$
reasoning, and final decision checks. Both paths write through
the same ledger interface, so their outputs remain comparable
under the same audit. The classifier rules, call budgets, and
prompts are given in Appendix~\ref{app:dispatcher_pipeline}.

\paragraph{Event-triggered typed repair.}
\label{sec:verify_repair}
Verification is implemented as an event handler over ledger
state rather than as always-on free-form self-critique. The
verifier fires on high-risk events such as tool anomaly, stale
reference, evidence conflict, confidence drop, unsupported
decision, or conclusion--evidence mismatch. Once triggered,
repair is restricted to seven typed operators in three layers:
evidence (\textsc{Drop}, \textsc{Refresh}), action
(\textsc{Retry}, \textsc{Switch}, \textsc{Acquire}), and
trajectory (\textsc{StopAndAnswer}, \textsc{Abstain}). The
policy follows locality-first escalation with a fixed
per-trigger budget, so repair modifies the ledger state or
invokes tools instead of rewriting the trajectory in
unconstrained natural language. Formal trigger conditions,
operator semantics, and cost bounds are given in
Appendix~\ref{app:triggers} and Appendix~\ref{app:repair}.

\subsection{Provenance Non-Amplification}
\label{sec:nonamp}

The typed repair interface yields a provenance-level guarantee.

\begin{proposition}[Provenance Non-Amplification]
\label{prop:nonamp}
For any repair operator $r\in\mathcal{R}$ and any ledger
$\mathcal{L}$, let $\mathcal{L}'=r(\mathcal{L})$. Every new
entry in $\mathcal{L}'\setminus\mathcal{L}$ has tool-produced
provenance; that is, it is either a direct tool output or the
image of a tool output under the deterministic template mapping
$\mathcal{M}$.
\end{proposition}

\begin{proof}
\textsc{Drop} creates no new entry. \textsc{Refresh},
\textsc{Retry}, and \textsc{Acquire} add entries only by
invoking a tool and applying $\mathcal{M}$ to the resulting
observation. \textsc{Switch} changes the action choice but
adds no ledger entry by itself. \textsc{StopAndAnswer} and
\textsc{Abstain} terminate without changing the ledger. Hence
any new entry must have tool-produced provenance.
\end{proof}

Proposition~\ref{prop:nonamp} does not assert that tool
outputs are always factually correct. Rather, it guarantees
that repair cannot fabricate new provenance-less content; any
remaining error is tied to an explicit source and therefore
remains auditable.

\subsection{Trajectory Faithfulness Metrics}
\label{sec:metrics}

Final-answer accuracy measures task success but not whether
the trajectory was supported by evidence. We therefore audit
trajectories with four metrics. A claim is grounded only if it
cites active ledger entries and passes the claim--evidence
containment checks of \S\ref{sec:grounding}. To prevent
perception-heavy pipelines from diluting the denominator with
tool-faithful OCs, the main unsupported-claim rate is computed
over SC and DC only. Let
$\mathcal{C}_\tau^{\mathrm{R}}=
\mathcal{C}_\tau^{\textsc{SC}}\cup\mathcal{C}_\tau^{\textsc{DC}}$,
and let $g(c)$ indicate that claim $c$ is grounded:
\vspace{-2mm}
\begin{align}
    \mathrm{UCR}_{\mathrm{reason}}(\tau)&=
    \frac{|\{c\in\mathcal{C}_\tau^{\mathrm{R}}:g(c)=0\}|}
    {|\mathcal{C}_\tau^{\mathrm{R}}|}
    \quad(\downarrow),
    \label{eq:metrics}\\[-1pt]
    \mathrm{GDR}(\tau)&=
    \frac{|\{c\in\mathcal{C}_\tau^{\textsc{DC}}:g(c)=1\}|}
    {|\mathcal{C}_\tau^{\textsc{DC}}|}
    \;(\uparrow).
\end{align}
For corpus-level diagnosis, with decision-grounding threshold $\gamma$,
\begin{align}
    \mathrm{R4R}&=
    \frac{|\{\tau:S(\tau)=1\wedge\mathrm{GDR}(\tau)>\gamma\}|}
    {|\{\tau:S(\tau)=1\}|}
    \;(\uparrow),
    \label{eq:r4r}\\[-1pt]
    \mathrm{WDG}&=
    \frac{|\{\tau:S(\tau)=0\wedge\mathrm{GDR}(\tau)>\gamma\}|}
    {|\{\tau:S(\tau)=0\}|}
    \;(\downarrow).
    \label{eq:wdg}
\end{align}
$\mathrm{UCR}_{\mathrm{reason}}$ and GDR measure unsupported
reasoning and grounded decisions. R4R separates
evidence-backed success from lucky correct answers, while WDG
captures wrong answers that nevertheless look grounded at the
decision level. We report OC error rate separately to isolate
perception noise, together with EUR, RR, RC, and SE in
Appendix~\ref{app:metrics_full}.

\vspace{-2mm}
\section{Experiments}
\vspace{-2mm}
\subsection{Experimental Setup}
\label{sec:setup}

We evaluate \logotitle{} on six public benchmarks---
VTC-Bench~\citep{zhu2026vtcbench},
MMStar~\citep{chen2024mmstar},
MMMU~\citep{yue2024mmmu},
MMMU-Pro~\citep{yue2025mmmupro},
EMMA~\citep{hao2025emma}, and
MC-Search~\citep{ning2026mcsearch}---plus the in-house Hard-200
stress set (Appendix~\ref{app:hard200}); trajectory-level
faithfulness is audited on V*Bench and EMMA-160. \logotitle{} is
instantiated on six frontier MLLMs from four vendors
(GPT-4o/GPT-5.5, Gemini-3-Flash/3.1-Pro,
Claude-Sonnet-4.6/Opus-4.7, Kimi-K2.6) and compared against each
backbone's native CoT or thinking-mode output under an identical
tool budget, so any gap reflects framework design rather than
extra compute. Besides task accuracy, we report the
trajectory-level metrics of \S\ref{sec:metrics}
($\mathrm{UCR}_{\text{reason}}$, GDR, R4R, WDG) and, for
MC-Search, the chain-alignment metrics HPS and RD; claim-level
auditing uses a fixed external judge (Gemini-3.1-Pro) under an
identical protocol on both sides.

\vspace{-2mm}
\subsection{Main Results}
\vspace{-2mm}
\label{sec:main_results}

We report absolute scores together with the gap to each
native backbone; per-category and per-backbone tables are
deferred to Appendix~\ref{fULL_RESULTS}.

\paragraph{VTC-Bench and general VLM benchmarks.}
On VTC-Bench (Figure~\ref{fig:vtcbench_combined}(a)),
\logotitle{} with Gemini-3-Flash reaches $58.9\%$, a new
state of the art over every proprietary tool-use and
general-purpose baseline; the same framework lifts
Gemini-3.1-Pro by $+11.8$, Gemini-3-Flash by $+12.4$, and
GPT-4o by $+23.3$ points, with the largest gain on the
weakest backbone, indicating backbone-agnostic improvement.
\logotitle{} also ranks first on MMStar, MMMU, and MMMU-Pro
against eleven baselines (Figure~\ref{fig:vtcbench_combined}(b));
the MMMU-Pro gain is particularly telling because multi-step
agents tend to
overthink~\citep{sui2025overthinking,chen2024overthinking}
on this setting, yet the Adaptive Dual-Path Dispatcher
(\S\ref{sec:dispatcher}) routes knowledge-heavy queries to
the direct path, addressing F3 by control policy rather than
longer chains of thought.

\begin{figure*}[t]
    \centering

    \includegraphics[width=0.98\textwidth]{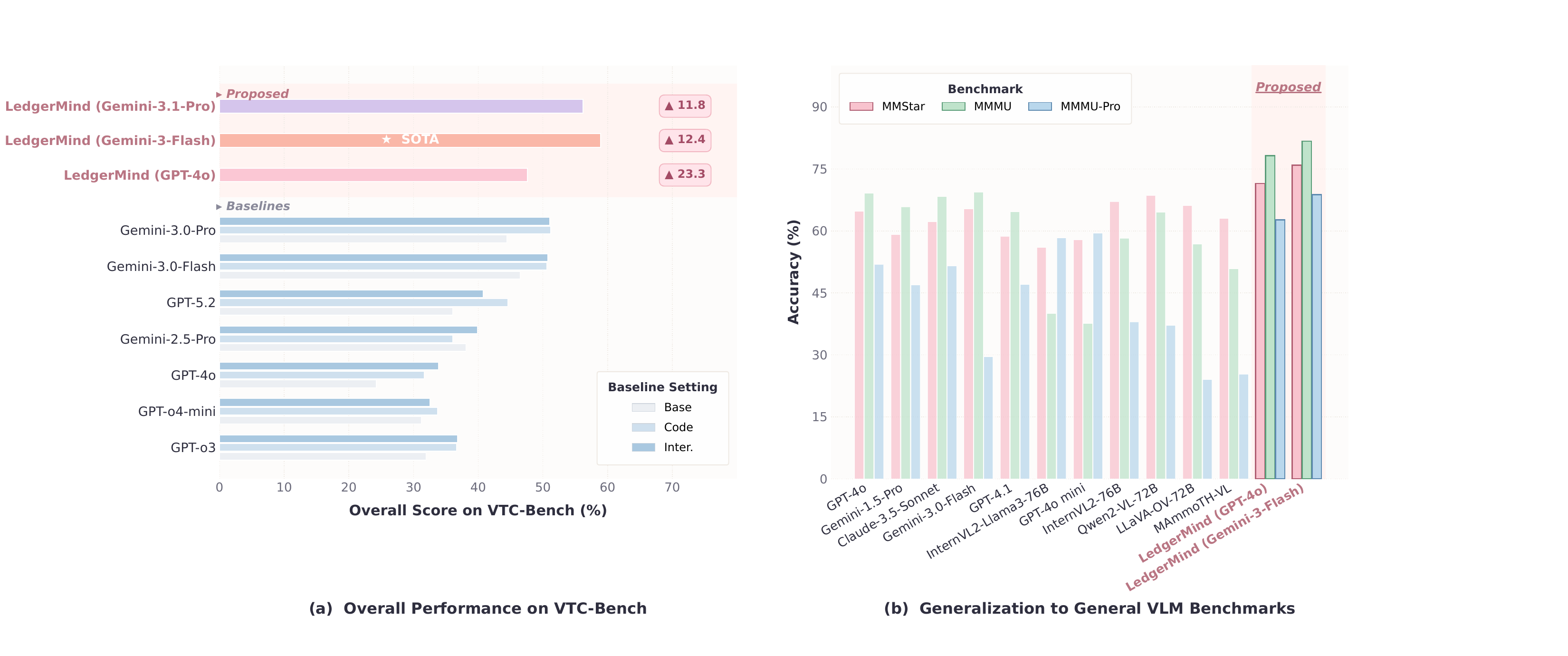}
    \caption{(a) Overall accuracy on
    VTC-Bench~\citep{zhu2026vtcbench}: \logotitle{} with
    Gemini-3-Flash reaches $58.9\%$, a new state of the art.
    (b) Generalisation to three general VLM benchmarks
    (MMStar~\citep{chen2024mmstar},
    MMMU~\citep{yue2024mmmu},
    MMMU-Pro~\citep{yue2025mmmupro}): \logotitle{} with GPT-4o
    and Gemini-3-Flash backbones outperforms 11
    baselines.}
    \label{fig:vtcbench_combined}

\end{figure*}

\paragraph{EMMA and Hard-200.}

On EMMA (Table~\ref{tab:emma_main}), \logotitle{} attains
$58.29\%$ overall ($+9.58$ over the strongest thinking-mode
baseline), with gains concentrated on Math ($+16.15$ pp) and
Physics ($+16.02$ pp)---disciplines that stress diagram
reading and multi-step symbolic derivation, and are therefore
most susceptible to F2 and F4. The $-0.18$ change on Coding
reflects that subset's multiple-choice
code--visualization alignment nature, which de-emphasizes
image-centric reasoning. On Hard-200
(Figure~\ref{fig:hard200_combined}), \logotitle{} improves
every backbone by $+11.2$ to $+19.7$ points overall, gains
hold on all three sub-sources, and the heatmap contains no
negative cell; the weakest backbone Kimi-K2.6 jumps from
$19.5\%$ to $46.0\%$ ($+26.5$) on BrowseComp-VL, the
behavior expected from entity-level grounding against
citation-backed spurious grounding (F2).

\begin{figure*}[t]
    \centering

    \includegraphics[width=0.98\textwidth]{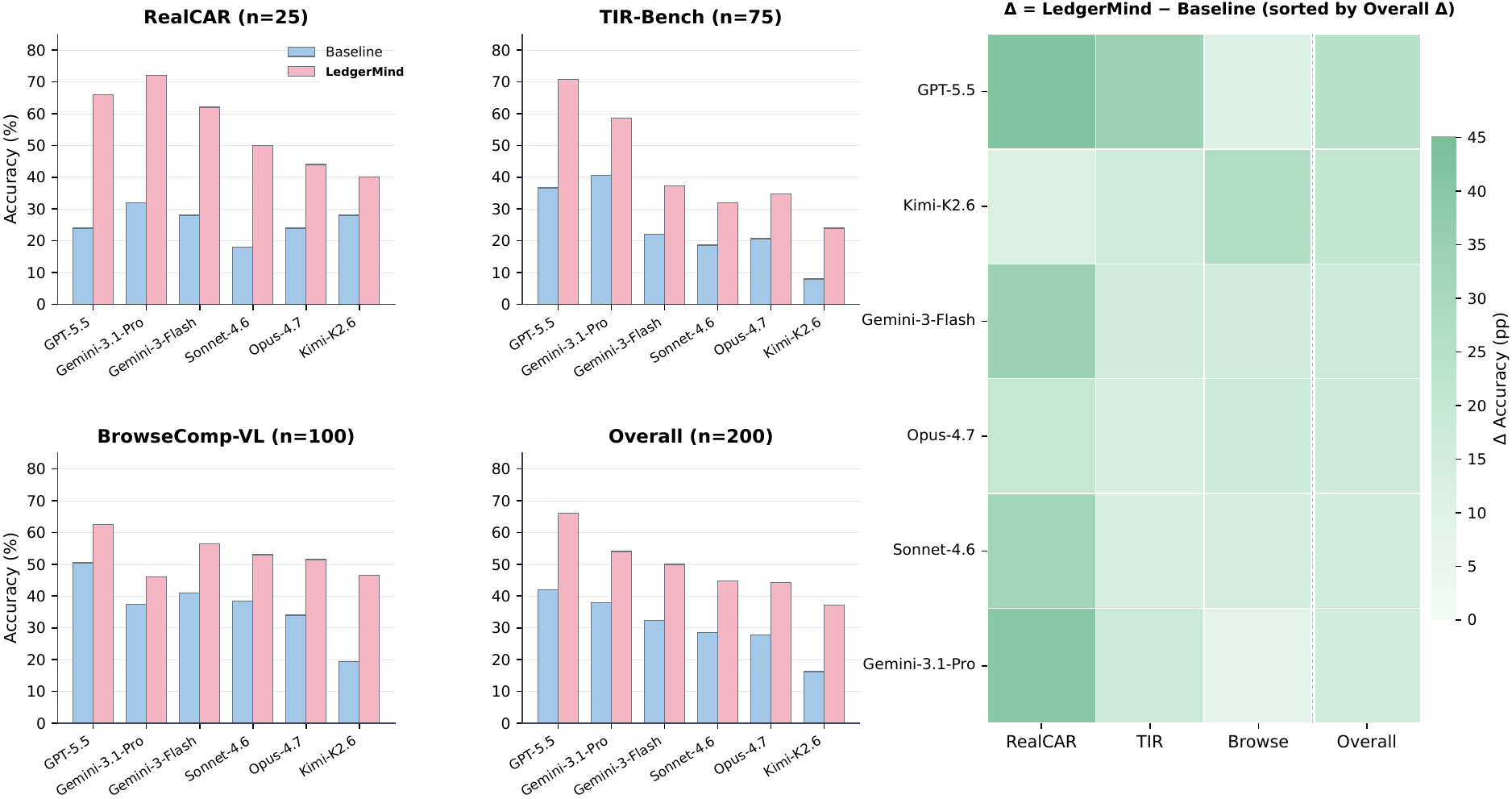}

    \caption{Results on Hard-200
    (Appendix~\ref{app:hard200}) across six frontier backbones.
    \logotitle{} improves every backbone on the overall score
    and on each of the three sub-sources (RealCAR, TIR-Bench,
    BrowseComp-VL); the right heatmap
    ($\Delta =$\,\logotitle{}$-$\,baseline) contains no
    negative cell.}
    \label{fig:hard200_combined}

\end{figure*}

\begin{table*}[t]
\centering
\small
\begin{tabular}{@{}p{0.515\textwidth}@{\hspace{0.015\textwidth}}p{0.46\textwidth}@{}}

\begin{minipage}[t]{\linewidth}
\vspace{0pt}
\centering
\begin{minipage}[t][3.2\baselineskip][t]{\linewidth}
\centering\footnotesize
\captionof{table}{Main results on EMMA~\citep{hao2025emma}. }
\label{tab:emma_main}
\end{minipage}

\footnotesize
\setlength{\tabcolsep}{3.0pt}
\renewcommand{\arraystretch}{1.10}
\resizebox{1\linewidth}{!}{
\begin{tabular}{lccccc}
\toprule
\textbf{Model} & \textbf{Math} & \textbf{Phys.} & \textbf{Chem.} & \textbf{Coding} & \textbf{Overall} \\
\midrule
Random choice
& 14.01 & 25.64 & 16.50 & 25.71 & 18.08 \\
\midrule
Claude 3.5 Sonnet
& 25.34 & 33.97 & 40.90 & 38.65 & 35.08 \\
GPT-4o
& 27.24 & 38.46 & 31.89 & 40.07 & 32.42 \\
Gemini 2.0 Flash
& 23.88 & 38.46 & 36.31 & 42.02 & 33.61 \\
Qwen2-VL-72B-Instruct
& 33.07 & 42.31 & 32.06 & 34.57 & 33.46 \\
InternVL2.5-78B
& 31.39 & 38.46 & 35.20 & 31.91 & 33.50 \\
Gemini 3.0 Flash
& 42.42 & 52.56 & 46.68 & 68.44 & 50.07 \\
Claude Opus 4.7
& 45.85 & 64.10 & 43.88 & 54.96 & 47.88 \\
Gemini 3.1 Pro
& 40.02 & 47.44 & 42.26 & 76.24 & 48.71 \\
\midrule
\logotitle{} (Ours)
& \textbf{56.17}
& \textbf{63.46}
& \textbf{50.68}
& \textbf{76.06}
& \textbf{58.29} \\
$\Delta$ (vs.\ Gemini 3.1 Pro)
& \textcolor{green}{$\uparrow 16.15$} & \textcolor{green}{$\uparrow 16.02$} & \textcolor{green}{$\uparrow 8.42$} & \textcolor{red}{$\downarrow 0.18$} & \textcolor{green}{$\uparrow 9.58$}\\
\bottomrule
\end{tabular}
}

\end{minipage}
&

\begin{minipage}[t]{\linewidth}
\vspace{0pt}
\centering
\begin{minipage}[t][3.2\baselineskip][t]{\linewidth}
\centering\footnotesize
\captionof{table}{Weighted-average comparison on
MC-Search~\citep{ning2026mcsearch}. G-F1: Golden F1.}
\label{tab:ledgermind_weighted_avg}
\end{minipage}

\footnotesize
\setlength{\tabcolsep}{2.0pt}
\renewcommand{\arraystretch}{1.32}
\resizebox{\linewidth}{!}{
\begin{tabular}{llccccc}
\toprule
\multirow{2}{*}{\textbf{Type}}
& \multirow{2}{*}{\textbf{Model}}
& \multicolumn{2}{c}{\textbf{Answer Acc.}}
& \multicolumn{2}{c}{\textbf{Chain Align.}}
& \multirow{2}{*}{\textbf{G-F1}$\uparrow$} \\
\cmidrule(lr){3-4} \cmidrule(lr){5-6}
&
& \textbf{F1}$\uparrow$
& \textbf{LJ}$\uparrow$
& \textbf{HPS}$\uparrow$
& \textbf{RD}$\downarrow$
& \\
\midrule
\textbf{Baseline}
& \textit{Best Official}
& \textit{41.78} & \textit{2.99} & \textit{31.35} & \textit{0.89} & \textit{67.52} \\
\midrule
\multirow{8}{*}{\logotitle{}}
& Claude-Opus-4.7    & \textbf{61.28} & \underline{4.08} & \textbf{57.82} & \textbf{0.54} & \textbf{76.58} \\
& GPT-5.5            & \underline{60.08} & \textbf{4.11} & \underline{55.88} & \underline{0.57} & 75.96 \\
& Claude-Sonnet-4.6  & 57.65 & 3.91 & 55.07 & 0.58 & 73.65 \\
& Gemini-2.5-Pro     & 54.48 & 3.54 & 46.22 & 0.71 & \underline{76.14} \\
& Gemini-3.1-Pro     & 52.01 & 3.52 & 48.46 & 0.69 & 70.87 \\
& GPT-4o-Mini        & 48.69 & 3.25 & 44.26 & 0.72 & 70.93 \\
& Gemini-2.5-Flash   & 44.14 & 3.10 & 41.77 & 0.78 & 68.38 \\
& Gemini-3-Flash     & 42.75 & 3.03 & 38.26 & 0.82 & 67.83 \\
\bottomrule
\end{tabular}
}
\end{minipage}
\end{tabular}

\end{table*}

\paragraph{Chain-aligned reasoning on MC-Search.}
To test whether gains come from genuinely grounded reasoning
rather than final-answer shortcuts, we evaluate on
MC-Search~\citep{ning2026mcsearch}, which scores both the
final answer (F1, LJ) and the intermediate retrieval chain
via HPS (fraction of golden steps recovered) and RD
(absolute deviation in step count), with a Golden F1 upper
bound when reference chains and retrieved content are both
provided. Table~\ref{tab:ledgermind_weighted_avg}
reports the weighted-average comparison across the official
topology distribution (topology-wise results in
Appendix~\ref{fULL_RESULTS}). \logotitle{} with
Claude-Opus-4.7 and GPT-5.5 lifts F1 from the best official
baseline's $41.78\%$ to $61.28\%/60.08\%$, raises HPS from
$31.35\%$ to $57.82\%/55.88\%$, and nearly halves RD from
$0.89$ to $0.54/0.57$. Chain-alignment gains exceed answer-F1 gains, and RD decreases while HPS
increases, which is the behavioral signature of grounded trajectories rather
than post-hoc final-answer correction, consistent with the
citation-only constraint of \S\ref{sec:ledger} and the
typed repair of \S\ref{sec:verify_repair}. Even under the
weakest Gemini-3-Flash backbone the chain metrics remain at
or above the best official baseline, indicating that the
framework, not the backbone, is the primary driver of
trajectory-level faithfulness.

\subsection{Trajectory-Level Faithfulness Diagnostics}
\label{sec:faithfulness_diag}

To check whether the gains of \logotitle{} come from grounded
reasoning rather than from coincidentally correct answers, we
run a Symmetric Reasoning Faithfulness Audit (S-RFA) that
instantiates the trajectory-level metrics of
\S\ref{sec:metrics}. For every sample we collect both a
baseline chain-of-thought trace and a \logotitle{} trace and
hand them to a fixed auditor (Gemini-3.1-Pro) that decomposes
each trace into at most ten atomic claims and labels each
claim's grounding against the image, yielding
$\mathrm{UCR}_{\text{reason}}$, GDR, R4R, and WDG under an
identical protocol on both sides. As shown in
Figure~\ref{fig:faithfulness_radar}, the \logotitle{} polygon
encloses the baseline on every backbone and every axis across
both V*Bench and EMMA-160: accuracy rises together with GDR
and R4R, while UCR and WDG both drop, so the accuracy lift is
earned by grounded trajectories rather than by post-hoc
rewriting. The shrinkage is largest on the $1-\mathrm{WDG}$
axis (most visible on EMMA-160), where baselines often produce wrong answers whose decision claims nevertheless
appear grounded,
precisely the Phantom Grounding pattern (F2,
\S\ref{sec:intro}) that the entity-level ECC of
\S\ref{sec:grounding} is built to intercept.

\begin{figure*}[t]
    \centering
    \includegraphics[width=0.99\textwidth]{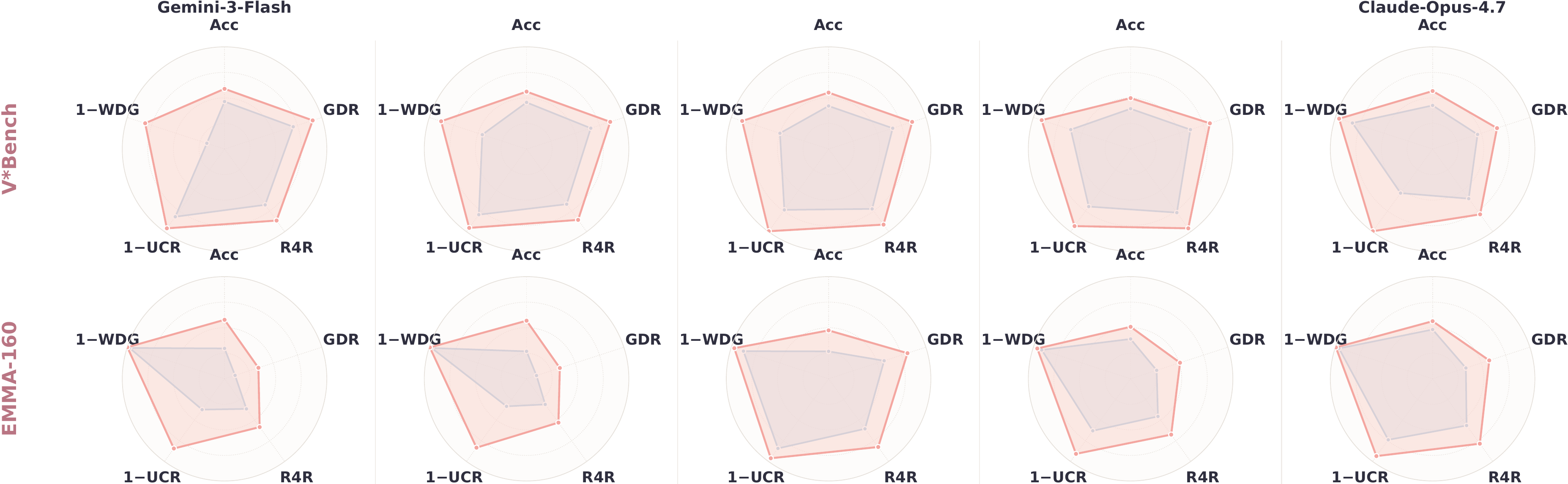}
    \caption{Symmetric Reasoning Faithfulness Audit on V*Bench
and EMMA-160 across five frontier MLLMs. \textcolor{baselineBlue}{Blue}
polygons denote baseline models, while \textcolor{ledgerMindPink}{pink}
polygons denote \logotitle{}. All axes are oriented such that larger
is better: Acc, GDR, R4R, $1-\mathrm{UCR}$, and $1-\mathrm{WDG}$.
The \logotitle{} polygon consistently encloses the baseline polygon
across benchmarks, backbones, and metrics.}
    \label{fig:faithfulness_radar}

\end{figure*}

\FloatBarrier

\subsection{Ablation Study}
\label{sec:ablation}

\begin{table}[H]
\centering
\footnotesize

\setlength{\tabcolsep}{5pt}
\renewcommand{\arraystretch}{1.1}
\caption{Ablation on MMMU-Pro with Gemini-3-Flash. All
variants share the same tool budget and decoding settings.
$\Delta$ is relative to the full framework on Overall.}

\label{tab:ablation}
\begin{tabular}{lccccr}
\toprule
\textbf{Variant} & \textbf{Easy} & \textbf{Medium} & \textbf{Hard} & \textbf{Overall} & $\bm{\Delta}$ \\
\midrule
\logotitle{} (full)         & \textbf{77.5} & \textbf{69.0} & \textbf{56.9} & \textbf{68.89} & --     \\
\midrule
w/o Ledger                  & 75.7 & 50.3 & 30.7 & 53.50 & $-15.39$ \\
w/o Typed Repair            & 76.9 & 60.3 & 38.9 & 60.40 & $-8.49$  \\
w/o ECC/NCC                 & 76.1 & 62.5 & 49.6 & 63.70 & $-5.19$  \\
w/o Dual-Read               & 76.7 & 67.2 & 52.4 & 66.70 & $-2.19$  \\
w/o Dispatcher              & 72.0 & 68.5 & 57.1 & 67.40 & $-1.49$  \\
\bottomrule
\end{tabular}

\end{table}

To show that each component of \logotitle{} is
\emph{individually necessary} rather than merely reasonable
in isolation, we ablate five variants on MMMU-Pro with
Gemini-3-Flash (Table~\ref{tab:ablation}), all sharing the
same tool budget so that any gap reflects the ablated design
choice rather than extra compute. Removing the Structured
Evidence Ledger is by far the most damaging ablation,
confirming that a prompt-level citation instruction cannot
substitute for a structural provenance constraint.
Replacing typed repair with free-form
self-reflection~\citep{shinn2023reflexion,madaan2023selfrefine}
is the second most harmful, consistent with repair-time
amplification (F4). Disabling entity- and numeric-level
grounding hurts the Hard split the most, identifying ECC/NCC
as the decisive mechanism against Phantom Grounding (F2)
rather than structural coverage alone, while Dual-Read
contributes primarily on measurement-heavy cases. Finally,
although forcing every query through \textsc{FullPipe}
changes the overall score only marginally, Easy accuracy
drops while Hard stays unchanged, the exact signature of the
Over-Reasoning Paradox (F3) and direct evidence that the
dispatcher's value lies in avoiding unnecessary depth on
simple queries rather than in adding it on hard ones.

\section{Conclusion}

We presented \logotitle{}, a training-free,
provenance-constrained runtime that uses a Structured Evidence
Ledger to coordinate claim grounding, adaptive execution, and
typed repair under a formal provenance non-amplification
guarantee. Guided by four
recurring failure patterns that aggregate accuracy obscures,
\logotitle{} consistently improves both answer accuracy and
trajectory-level faithfulness across five answer-level
multimodal benchmarks (VTC-Bench, MMStar, MMMU, MMMU-Pro,
EMMA), one chain-level multimodal search benchmark
(MC-Search), and the in-house Hard-200 stress set, over multiple
backbone MLLMs from four vendors; a symmetric reasoning audit further confirms
that the gains stem from grounded trajectories rather than
post-hoc answer rewriting. Future directions include
extending the ledger to long-horizon agents with persistent
memory, generalising the typed-repair machinery to other
modalities such as video and embodied interaction, and
exploring how ledger-level supervision can be turned into a
training signal for MLLMs.

\bibliographystyle{iclr2026_conference}
\bibliography{references}

\clearpage

\appendix

\section{Structured Reasoning Pipeline: Full Specification}
\label{app:algorithm}

This appendix gives the full specification of the reasoning
procedure run on the \textsc{FullPipe} branch of the Adaptive
Dual-Path Dispatcher (\S\ref{sec:dispatcher}). The procedure has
three stages: evidence gathering, grounded reasoning, and
decision with defense. Every branch mentioned in the main text
(Dual-Read Verification, two-round knowledge search, the
grounding cascade, hypothesis verification, and the three
defense lines) is expanded here.

\paragraph{Evidence gathering.}
The agent first produces a task plan from the question text
alone, recording observation targets, a reasoning method, and an
answer-granularity constraint (Appendix~\ref{app:granularity});
the plan is registered as a \textsc{Derivation} entry with
confidence $\sigma=1.0$. The agent then extracts up to
$N_{\max}$ Observation Claims from the image, each tagged with
a semantic category (Appendix~\ref{app:oc_categories}) and
registered as \textsc{Perception} evidence with $\sigma=0.92$.
The text pool $\mathcal{P}_\text{oc}$ used for ECC and NCC is
built from the \textsc{Text} and \textsc{Number} categories.
For reading or localization questions, the agent selects the
most informative region from a nine-zone grid, applies a
$2.5\times$ crop-zoom, and replaces the lowest-confidence
perception entry via \textsc{Supersede}. Reading questions
additionally trigger Dual-Read Verification
(Appendix~\ref{app:dualread}), which raises crop confidence to
$0.97$ on agreement and lowers it to $0.65$ on disagreement.
For knowledge-oriented questions, a first search round uses
OC-derived entities; deep-knowledge questions trigger a second
round whose query is refined using the first-round results.
All search results are appended as \textsc{Retrieval} entries
with $\sigma=0.85$.

\paragraph{Grounded reasoning.}
State Claims are generated in a three-stage
$[\text{E}]\!\to\![\text{I}]\!\to\![\text{J}]$ chain. Every
non-initial claim must cite ledger entry IDs. Non-judgment
claims with an empty citation set are dropped; judgment claims
with an empty citation set are kept but degraded to
$\sigma=0.55$. Each judgment then passes through the grounding
cascade of \S\ref{sec:grounding}: support coverage
(Eq.~\ref{eq:rho}), entity-level ECC (Eq.~\ref{eq:ecc}), and
NCC (Eq.~\ref{eq:ncc}), with confidence lowered to $0.50$ on
ECC failure and to $0.52$ on NCC failure, both below the
verification threshold $\sigma_{\mathrm{verify}}=0.6$. If all
judgments end up weak ($\sigma \le \sigma_{\mathrm{verify}}$)
and the repair budget allows, the strongest judgment is treated
as a falsifiable hypothesis: ECC failure triggers a
re-examination restricted to $\mathcal{P}_\text{oc}$ and
preceded by an explicit hallucination warning, while NCC
failure triggers a second independent reading of the crop with
the value closest to $\mathcal{P}_\text{oc}$ adopted.

\paragraph{Decision and defense.}
The Decision Claim is synthesised with the images re-transmitted
alongside the ledger, which lets the agent visually verify the
reasoning chain at decision time. Three defense lines are then
applied to the final answer. Entity recheck forces a re-answer
restricted to $\mathcal{P}_\text{oc}$ when ECC fails on
$\hat{y}$. Numeric recheck, active for reading questions, snaps
the answer to the value in $\mathrm{Num}(\mathcal{P}_\text{oc})$
closest to $v_{\hat{y}}$. Visual fallback is triggered when
$\hat{y}$ contains one of the uncertainty markers in
Appendix~\ref{app:uncertainty}; in that case the agent is
forced to answer directly from the image with a ``give a
concrete answer'' instruction. The final output is the pair
$(\hat{y}, (\mathcal{L}, \tau))$.

\begin{algorithm*}[t]
\caption{\logotitle{}: Structured Reasoning Pipeline (full
version, Part 1 of 2 --- evidence gathering and grounded
reasoning).}
\label{alg:pipeline_full}
\small
\begin{algorithmic}[1]
\Require Question $q$, images $\mathcal{I}$, repair budget
         $M=2$
\Ensure Final answer $\hat{y}$, audit trail
        $(\mathcal{L}, \tau)$ (returned in Part~2)
\State $\mathcal{L} \gets \emptyset$;\;
       $\mathcal{P}_\text{oc} \gets \emptyset$
\Statex \textcolor{gray}{\hrulefill\; \textit{Evidence gathering}
        \;\hrulefill}
\State $\pi \gets \textsc{TaskPlan}(q)$;\;
       $\mathcal{L}.\textsc{Append}(\pi,\textsc{Derivation},1.0)$
\State $\{e_i\} \gets \textsc{Observe}(q,\mathcal{I},\pi)$;
       append each as \textsc{Perception} with $\sigma=0.92$
\State $\mathcal{P}_\text{oc} \gets
       \bigcup_{\gamma_i\in\{\textsc{T},\textsc{N}\}} f_{e_i}$
\If{$q \in \mathcal{Q}_\text{read}\cup\mathcal{Q}_\text{location}$}
    \State $r^* \gets$ most informative region (9-zone grid)
    \State $e_c \gets \textsc{CropZoom}(\mathcal{I},r^*,z{=}2.5)$
    \State $\mathcal{L}.\textsc{Supersede}(
           \arg\min_{e:\kappa(e)=\textsc{Perc}}\sigma_e,\,e_c)$
    \If{$q \in \mathcal{Q}_\text{read}$}
        \State $(v_1,v_2) \gets$ two independent readings of
               $e_c$
        \State $\Delta_v \gets |v_1-v_2|/\max(v_1,v_2)$
        \State $\sigma(e_c) \gets 0.97$ if $\Delta_v\le 0.15$
               else $0.65$; adopt $v_2$ on disagreement
    \EndIf
\EndIf
\If{$q \in \mathcal{Q}_\text{knowledge}$}
    \State $\mathcal{E}_s^{(1)} \gets \textsc{WebSearch}(
           \textsc{ExtractQuery}(q,\mathcal{P}_\text{oc}))$
    \If{$q \in \mathcal{Q}_\text{deep\_knowledge}$}
        \State $\mathcal{E}_s^{(2)} \gets \textsc{WebSearch}(
               \textsc{RefineQuery}(q,\mathcal{E}_s^{(1)}))$
    \EndIf
    \State append $\mathcal{E}_s$ as \textsc{Retrieval} with
           $\sigma=0.85$
\EndIf
\Statex \textcolor{gray}{\hrulefill\; \textit{Grounded reasoning}
        \;\hrulefill}
\State $\{c_j\} \gets \textsc{Reason}(\mathcal{L},q)$ via
       $[\textbf{E}]\!\to\![\textbf{I}]\!\to\![\textbf{J}]$
\For{each SC $c_j$}
    \State $\mathrm{ground}(c_j) \gets$ resolve citations,
           require $\omega=\textsc{Active}$
    \If{$\mathrm{type}(c_j)\in\{[\textbf{E}],[\textbf{I}]\}$
        and $\mathrm{ground}(c_j)=\emptyset$}
        \State $\mathcal{L}.\textsc{Drop}(c_j)$
    \ElsIf{$\mathrm{type}(c_j)=[\textbf{J}]$ and
           $\mathrm{ground}(c_j)=\emptyset$}
        \State $\sigma(c_j) \gets 0.55$
    \EndIf
\EndFor
\For{each $[\textbf{J}]$-type claim $c$}
    \State compute $\rho(c)$ (Eq.~\ref{eq:rho}),
           $\mathrm{ECC}(c)$ (Eq.~\ref{eq:ecc}),
           $\mathrm{NCC}(c)$ (Eq.~\ref{eq:ncc})
    \State \textbf{if} $\mathrm{ECC}(c)=0$:
           $\sigma(c)\gets 0.50$; \textbf{elif}
           $\mathrm{NCC}(c)=0$: $\sigma(c)\gets 0.52$
\EndFor
\If{$\forall c\in\mathcal{J}:\sigma(c)\le
    \sigma_{\mathrm{verify}}=0.6$ and $M>0$}
    \State $c^* \gets \arg\max_{c\in\mathcal{J}}\sigma(c)$
    \If{$\mathrm{ECC}(c^*)=0$}
        \State inject hallucination warning; $c' \gets
               \textsc{ReExamine}(\mathcal{I},
               \mathcal{P}_\text{oc})$
        \State \textbf{if} $c'\neq c^*$: append $c'$ as
               \textsc{Derivation} with $\sigma=0.88$
    \ElsIf{$\mathrm{NCC}(c^*)=0$}
        \State $v' \gets$ second independent reading; adopt
               value closest to
               $\mathrm{Num}(\mathcal{P}_\text{oc})$
    \EndIf
    \State $M \gets M-1$
\EndIf
\State \textbf{continue to Part~2
       (Algorithm~\ref{alg:pipeline_full_p2})}
\end{algorithmic}
\end{algorithm*}

Part~2 of the pipeline performs the final decision and the
three defense lines described above. At this point the ledger
$\mathcal{L}$ contains all \textsc{Perception},
\textsc{Retrieval}, and \textsc{Derivation} entries produced
by the evidence-gathering and grounded-reasoning stages, and
the repair budget $M$ is what remains after any hypothesis
verification.

\begin{algorithm*}[t]
\caption{\logotitle{}: Structured Reasoning Pipeline (full
version, Part 2 of 2 --- decision and defense).}
\label{alg:pipeline_full_p2}
\small
\begin{algorithmic}[1]
\Statex \textcolor{gray}{\hrulefill\; \textit{Decision and
        defense} \;\hrulefill}
\State $\hat{y} \gets \textsc{Decide}(\mathcal{I},
       \{e_i\}_\text{OC},\{c_j\}_\text{SC},e_c,\mathcal{E}_s)$
       \Comment{images re-sent}
\If{$\mathrm{ECC}(\hat{y},\mathcal{P}_\text{oc})=0$}
    \State $\hat{y} \gets
           \textsc{ReAnswer}(q,\mathcal{P}_\text{oc},
           \text{``only use visible text''})$
    \Comment{entity recheck}
\EndIf
\If{$\mathrm{NCC}(\hat{y},\mathcal{P}_\text{oc})=0$ and
    $q\in\mathcal{Q}_\text{read}$}
    \State $\hat{y} \gets \arg\min_{v\in
           \mathrm{Num}(\mathcal{P}_\text{oc})}|v-v_{\hat{y}}|$
    \Comment{numeric recheck}
\EndIf
\If{$\hat{y}\in\mathcal{U}$}
    \State $\hat{y} \gets \textsc{DirectVisual}
           (q,\mathcal{I},$
    \Statex \hspace{\algorithmicindent}
           $\text{``give a concrete answer''})$
    \Comment{visual fallback}
\EndIf
\State \Return $\hat{y}$, audit trail $(\mathcal{L},\tau)$
\end{algorithmic}
\end{algorithm*}

\FloatBarrier

\section{Evidence Ledger: Full Schema and Operations}
\label{app:ledger_full}

This appendix expands the 11-field entry schema and the four
lifecycle operations introduced in \S\ref{sec:ledger}. The
dependency graph supports the localized re-checking mentioned
in the main text, so that when an entry changes status only
the affected claims need to be re-evaluated.

\subsection{Evidence Entry Schema}
\label{app:schema}

Each entry $e \in \mathcal{L}_t$ is the following 11-field
tuple:
\begin{equation}
    e = \bigl(\mathrm{id},\mathrm{src},\kappa,t_e,f_e,\sigma_e,
    \omega_e,b_e,\ell_e,\xi_e,\mathcal{D}_e\bigr),
    \label{eq:entry_full}
\end{equation}
where $id$ is a unique identifier, $src$ the tool, $\kappa$ the epistemic
type, $t_e$ the creation timestamp, $f_e$ the rule-normalized fact given by
$f_e=\mathcal{M}(o_t)$, $\sigma_e\in[0,1]$ the confidence, $\omega_e$ the
lifecycle status, $b_e$ an optional spatial bounding box, $\ell_e$ the
time-to-live parameter specified by the TTL policy, $\xi_e$ the superseding
entry ID, and $D_e$ the reverse dependency set of claims that cite $e$.

\paragraph{Template mapping examples.}
The deterministic mapping
$\mathcal{M}: \mathcal{O} \to \Sigma^*$ is tool-specific:
\begin{center}
\small
\begin{tabular}{lll}
\toprule
\textbf{Tool} & \textbf{Fact $f_e$} & \textbf{Metadata} \\
\midrule
OCR & recognized\_text & $b_e =$ detection\_bbox \\
Web search & snippet & $\mathrm{src} =$ URL \\
Crop & crop\_metadata & $b_e =$ crop\_region \\
VQA & model\_answer & $\kappa =$ \textsc{Derivation} \\
\bottomrule
\end{tabular}
\end{center}

\subsection{Lifecycle Operations}
\label{app:ops}

The four operations form a finite state machine over
$\Omega = \{\textsc{Active}, \textsc{Stale},
\textsc{Conflicted}, \textsc{Dropped}\}$:
\begin{align}
    &\textsc{Append}(e_{\text{new}}): \nonumber\\[-2pt]
        & \mathcal{L}_{t+1} = \mathcal{L}_t \cup
          \{e_{\text{new}}\},\;
          \omega_{e_{\text{new}}} \gets \textsc{Active},
          \label{eq:append_full}\\
    &\textsc{MarkStale}(e): \nonumber\\[-2pt]
        & \omega_e \gets \textsc{Stale} \;\;\text{iff}\;\;
          (t_{\text{curr}} - t_e) > \ell_e, \nonumber\\
        & \forall c \in \mathcal{D}_e:\; \text{trigger
          re-verification of } c,
          \label{eq:stale_full}\\
    &\textsc{Supersede}(e_{\text{old}},e_{\text{new}}): \nonumber\\[-2pt]
        & \omega_{e_{\text{old}}} \gets \textsc{Stale},\;
          \xi_{e_{\text{old}}} \gets \mathrm{id}(e_{\text{new}}),
          \nonumber\\
        & \mathcal{L}_{t+1} = \mathcal{L}_t \cup
          \{e_{\text{new}}\},
          \label{eq:supersede_full}\\
    &\textsc{Drop}(e): \nonumber\\[-2pt]
        & \omega_e \gets \textsc{Dropped},\;
          \nonumber\\[-2pt]
        & \forall c \in \mathcal{D}_e:\; \text{propagate
          invalidation}.
          \label{eq:drop_full}
\end{align}

\subsection{Dependency Graph}
\label{app:depgraph}

The bipartite graph
$G_t = (V_E, V_C, E_{\mathrm{link}})$ is built as
$V_E = \{\mathrm{id}(e) : e \in \mathcal{L}_t\}$,
$V_C = \{\mathrm{id}(c) : c \in \bigcup_{t'}
\mathcal{C}_{t'}\}$, and $(e_i, c_j) \in E_{\mathrm{link}}
\iff \mathrm{id}(e_i) \in \mathrm{ground}(c_j)$. The localized
impact set satisfies $|\mathrm{Affected}(e)| \le
\Delta_{\max} \ll |\mathcal{C}_\tau|$, where $\Delta_{\max}$
is the maximum evidence-node degree. When $\omega_e$ changes,
only the claims in $\mathrm{Affected}(e)$ need to be
re-audited.

\section{Grounding Threshold Sensitivity}
\label{app:thresholds}

The confidence degradation values ($0.50$ for ECC failure,
$0.52$ for NCC failure, $0.55$ for structural ungrounding)
were chosen by grid search over
$\{0.40, 0.45, 0.50, 0.55, 0.60\}$ on a held-out development
set of 50 questions. The key requirement is that all values
fall below the verification threshold
$\sigma_{\mathrm{verify}} = 0.6$ so that hypothesis
verification is triggered. Using separate thresholds for ECC
and NCC gives a useful diagnostic signal: when both fail at
the same time, the lower ECC threshold causes entity-level
repair to take priority over numeric repair.

For NCC, the tolerance is type-aware. Years, dates, counts, option labels, and
identifiers use exact matching after normalization. The relative tolerance
$\delta_{\mathrm{read}}=0.15$ is used only for continuous visual measurements,
such as thermometer or gauge readings, to absorb small perceptual noise
(e.g., 36.7 versus 36.8). Performance is stable for
$\delta_{\mathrm{read}}\in[0.10,0.20]$ on measurement-style questions and
degrades outside this range. Thus the numeric check remains exact for symbolic
or categorical numbers while remaining tolerant to visual-reading noise.

\section{Dispatcher and Pipeline Details}
\label{app:dispatcher_pipeline}

\subsection{Complexity Classification}
\label{app:classify}

The classifier $\phi(q)$ in Eq.~\ref{eq:dispatch} routes a
question to \texttt{complex} if at least one of the following
conditions holds, and to \texttt{simple} otherwise.

\paragraph{Condition 1: precise perception.}
The question contains keywords indicating precise visual
reading, counting, measurement, or localization. English:
\texttt{``how many''}, \texttt{``count''},
\texttt{``temperature''}, \texttt{``thermometer''},
\texttt{``reading''}, \texttt{``measurement''},
\texttt{``gauge''}, \texttt{``scale''}, \texttt{``meter''},
\texttt{``locate''}, \texttt{``identify the region''}.

\paragraph{Condition 2: visual + external knowledge fusion.}
The answer requires combining visual evidence with external
facts. English: \texttt{``this company''},
\texttt{``this game''}, \texttt{``this film''},
\texttt{``this person''}, \texttt{``this artist''},
\texttt{``this album''}, \texttt{``this university''},
\texttt{``according to''}, \texttt{``what year''},
\texttt{``was the first''}, \texttt{``which country''},
\texttt{``founded''}, \texttt{``how many students''}.
Chinese (pinyin): \texttt{na bu} (which), \texttt{shenme pinpai}
(what brand), \texttt{shi shei} (who), \texttt{shenme donghua}
(what animation), and \texttt{shenme zhiwu} (what plant).

\paragraph{Condition 3: verification or multi-entity
reasoning.}
The question asks to verify a historical event, identify an
artifact, compare multiple entities, resolve conflicting
evidence, or find differences. Keywords include
\texttt{``verify''}, \texttt{``identify the artifact''},
\texttt{``historical event''}, \texttt{``compare''},
\texttt{``spot the difference''}, as well as any question
containing both a visual identification term and an external
fact-check term.

\paragraph{Condition 4: MCQ with visual grounding.}
A question is MCQ if the regex
\texttt{/\textbackslash b([A-F])[.)\textbackslash s]+\textbackslash S/}
matches at least two distinct option letters.

\paragraph{Condition 5: non-English.}
Let $\alpha(q) = |\{c \in q : c \in
\text{ASCII}_\text{alpha}\}| / |q|$. If $\alpha(q) \le 0.5$,
the question is treated as non-English and routed to
\texttt{complex}.

\paragraph{Deep-knowledge sub-keywords
($\mathcal{Q}_{\text{deep\_knowledge}}$).}
Within Condition 2, the following markers additionally
trigger a second round of knowledge search:
\texttt{``according to''}, \texttt{``based on''},
\texttt{``historically''}, \texttt{``originally''}.

\subsection{Answer Granularity Constraints}
\label{app:granularity}

\begin{table*}[t]
\centering
\small
\caption{Dynamic granularity constraints $\mathcal{F}(q)$ by
question type.}
\begin{tabular}{ll}
\toprule
\textbf{Question Type} & \textbf{Constraint} \\
\midrule
MCQ & ``Output only one option letter (A/B/C/D/E/F)'' \\
Instrument reading ($\mathcal{Q}_\text{read}$) &
    ``Precise to smallest graduation, preserve all decimals'' \\
Location ($\mathcal{Q}_\text{location}$) &
    ``Most specific toponym / building / landmark name'' \\
Spot-the-difference ($\mathcal{Q}_\text{diff}$) &
    ``List all differences dimension by dimension'' \\
Character counting ($\mathcal{Q}_\text{char}$) &
    ``Locate each character individually, then count'' \\
Other & ``Answer must be as specific as possible'' \\
\bottomrule
\label{app:D2}
\end{tabular}
\end{table*}

Given in Table~\ref{app:D2}.

\subsection{OC Semantic Categories}
\label{app:oc_categories}

\begin{table*}[t]
\centering
\small
\caption{OC category definitions and extraction instructions.}
\begin{tabular}{lp{9.5cm}}
\toprule
\textbf{Category} & \textbf{Instruction} \\
\midrule
\textsc{Text} &
    Transcribe all visible text character by character \\
\textsc{Number} &
    Extract all numerical values with units and full decimal
    precision \\
\textsc{Object} &
    Describe objects: appearance, color, shape, material,
    quantity \\
\textsc{Spatial} &
    Describe positional / arrangement relationships between
    objects \\
\textsc{Compare} &
    For comparison tasks: extract quantifiable attributes per
    object \\
\textsc{Count} &
    Scan zone by zone; assign a unique ID to each counted
    target \\
\textsc{Diff} &
    Compare two images along 5 dimensions: object presence,
    color, shape, position, text \\
\bottomrule
\end{tabular}
\label{app:D3}
\end{table*}

Given in Table~\ref{app:D3}.

\subsection{Reasoning Method Adaptation}
\label{app:reason_method}

\begin{table*}[t]
\centering
\small
\caption{Dynamic reasoning method $\mathcal{M}(q)$ adopted in
grounded reasoning.}
\begin{tabular}{ll}
\toprule
\textbf{Question Type} & \textbf{Method} \\
\midrule
MCQ & Per-option analysis with elimination \\
Counting & Verify each numbered target in OCs; confirm total \\
Comparison & Extract attribute values per object; compare
pairwise \\
Reading & Prefer crop evidence; read the precise value \\
Location & Prefer \textsc{Text}-category OCs for
identification \\
Find-difference & Dimension-wise listing from \textsc{Diff}
OCs \\
\bottomrule
\end{tabular}
\label{app:D4}
\end{table*}

Given in Table~\ref{app:D4}.

\subsection{Dual-Read Verification}
\label{app:dualread}

For reading tasks ($\mathcal{Q}_\text{read}$), after the
$2.5\times$ crop-zoom we issue two independent readings
$v_1, v_2$ of the same crop using two API calls that do not
share context. Let $\Delta_v=\frac{|v_1-v_2|}{\max(|v_1|,|v_2|,\epsilon)}$. We set
$\sigma(e_c) = 0.97$ if $\Delta_v \le 0.15$ (two readings
agree) and $\sigma(e_c) = 0.65$ otherwise (disagreement
triggers hypothesis verification). On disagreement we adopt
$v_2$, because in our trajectories an independent replication
is empirically more reliable than the first reading. Dual-Read
Verification therefore provides a second, perception-level
line of defense against the numeric sub-case of Phantom
Grounding.

\subsection{Knowledge Search Strategy}
\label{app:search}

\paragraph{Dual-engine redundancy.}
Primary: Serper.dev (Google Search API); fallback: SerpAPI
(activated on timeout or error from the primary). Both use
exponential backoff with jitter (at most 5 retries, 32~s
cap).

\paragraph{Result structure.}
Each search call returns three components: Answer Box (direct
factual answer), Knowledge Graph (structured entity
information), and Organic Results (top-3 snippets with URLs).
The three parts are concatenated and truncated to 600
characters for the re-answering prompt.

\paragraph{Two-round strategy.}
Round~1 builds the query from OC entities and question
keywords. Round~2 is only run for deep-knowledge queries,
using Round~1 results to refine the query.

\subsection{Uncertainty Markers}
\label{app:uncertainty}

The set $\mathcal{U}$ that triggers visual fallback:

\begin{center}
\small
\begin{tabular}{lp{0.67\columnwidth}}
\toprule
\textbf{Language} & \textbf{Markers} \\
\midrule
English & ``cannot determine'', ``cannot tell'', ``unable'',
          ``not sure'', ``unclear'' \\
System  & ``api\_error'', \texttt{unknown}, \texttt{need\_more} \\
\bottomrule
\end{tabular}
\end{center}

\section{Event Triggers and Typed Repair}
\label{app:verify_repair_full}

\subsection{Trigger Conditions}
\label{app:triggers}

Let $s_t = (a_t, o_t, \mathcal{C}_t)$ be the current step. The
verifier fires whenever any of the following holds:
\begin{align}
    \textbf{T1:}\;\; & o_t = \emptyset \;\vee\;
    \texttt{``error''} \in o_t
    \label{eq:t1}\\
    \textbf{T2:}\;\; & \exists\, \mathrm{id} \in
    \mathrm{ground}(c_t):
    t_{\text{curr}}-t_{\mathrm{id}}>\ell_{\mathrm{id}}
    \label{eq:t2}\\
    \textbf{T3:}\;\; & \exists\, e, e' \in
    \mathcal{L}_t^{\textsc{A}}:
    \mathrm{IoU}(b_e,b_{e'})>0.5
    \;\wedge\; f_e\neq f_{e'}
    \label{eq:t3}\\
    \textbf{T4:}\;\; & \sigma(e_t) <
    \sigma_{\mathrm{floor}}
    \label{eq:t4}\\
    \textbf{T5:}\;\; & \mathrm{type}(c_t) = \textsc{DC}
    \;\wedge\; \rho(c_t) < \beta
    \label{eq:t5}\\
    \textbf{T6:}\;\; & \rho(c_t) > 0 \;\wedge
    (\mathrm{ECC}(c_t)=0\;\vee\;\mathrm{NCC}(c_t)=0)
    \label{eq:t6}
\end{align}
T1--T6 denote Tool Anomaly, Stale Reference, Conflict,
Confidence Drop, Unsupported Decision, and Phantom
Grounding, respectively.
with decision coverage threshold $\beta = 0.15$, hard
confidence floor $\sigma_{\mathrm{floor}} = 0.3$ (used only
in T4), and grounding verification threshold
$\sigma_{\mathrm{verify}} = 0.6$ (used in the main-text
cascade of \S\ref{sec:grounding}). These two thresholds are
distinct: $\sigma_{\mathrm{floor}}$ marks an outright tool or
evidence failure, while $\sigma_{\mathrm{verify}}$ marks a
grounding-level weakness that calls for hypothesis
verification.

Once triggered, the verifier runs a four-step protocol:
(1)~a grounding audit that computes $\rho(c)$ for all
$c \in \mathcal{C}_t$;
(2)~a conflict scan that looks for contradictions in
$\mathcal{L}_t^{\textsc{Active}}$;
(3)~a staleness scan that checks TTL for all cited entries;
(4)~a repair recommendation that selects an operator using
the policy below.

\subsection{Repair Operator Definitions}
\label{app:repair}

The operator set is
$\mathcal{R} = \mathcal{R}_E \cup \mathcal{R}_A \cup
\mathcal{R}_T$:
\begin{align}
    \mathcal{R}_E &= \{\textsc{Drop}(e),\;
    \textsc{Refresh}(e, \theta')\}, \\
    \mathcal{R}_A &= \{\textsc{Retry}(a, \theta'),\;
    \textsc{Switch}(a \!\to\! a'),\;
    \textsc{Acquire}(q_{\text{spec}})\}, \\
    \mathcal{R}_T &= \{\textsc{StopAndAnswer}(e^*),\;
    \textsc{Abstain}\}.
\end{align}
These sets act at the evidence, action, and trajectory
layers, respectively.

\paragraph{Operator semantics.}
\textsc{Drop}$(e)$ sets $\omega_e \gets \textsc{Dropped}$,
creates no new entries, and propagates invalidation to every
claim in $\mathcal{D}_e$. \textsc{Refresh}$(e, \theta')$
marks $e$ as \textsc{Stale} and issues a new tool call with
updated parameters $\theta'$ (e.g., a different crop region
or a refined search query); the result enters via
\textsc{Supersede}$(e, e_{\text{new}})$.
\textsc{Retry}$(a, \theta')$ re-invokes the most recent
failed action with modified arguments: bbox perturbation
$\pm[15, 40]$ px for visual tools, query expansion with
extra keywords for search. \textsc{Switch}$(a \to a')$ swaps
the tool via a predefined compatibility map
$\mathcal{M}_{\text{compat}}$ (OCR $\to$ \{search, VQA\};
local crop $\to$ \{full-image VQA\}).
\textsc{Acquire}$(q_{\text{spec}})$ proactively gathers a
new piece of evidence to close a specific grounding gap for
a Decision Claim. \textsc{StopAndAnswer}$(e^*)$ terminates
execution with $\hat{y} = f_{e^*}$ and $e^* =
\arg\max_{e \in \mathcal{L}^{\textsc{A}}} \sigma_e$.
\textsc{Abstain} outputs $\hat{y} = \bot$ when evidence
remains insufficient after the repair budget is exhausted.

\paragraph{Repair policy.}
The mapping
$\Pi: \{T_1, \dots, T_6\} \to \mathcal{R}$ follows a
locality-first escalation:
\begin{align}
    \Pi(T_1) &= \textsc{Retry}\to\textsc{Switch}, \nonumber\\
    \Pi(T_2) &= \textsc{Refresh}\text{ (if DC dep.)}
                 \;/\;\textsc{Drop}, \nonumber\\
    \Pi(T_3) &= \textsc{Drop}\!\left(
        \arg\min_{e\in\mathrm{conflict}}\sigma_e\right),
        \nonumber\\
    \Pi(T_4) &= \textsc{Retry}(\theta'=\theta+\Delta\theta),
        \nonumber\\
    \Pi(T_5) &= \textsc{Acquire}\text{ (if }M>0\text{)}
                 \;/\;\textsc{Abstain}, \nonumber\\
    \Pi(T_6) &= \text{warning injection}
                 +\textsc{Acquire}.
    \label{eq:policy_full}
\end{align}

\paragraph{Repair budget and cost bound.}
Each trigger allows up to $M_0 = 2$ repairs. The total
trajectory cost is
\begin{equation}
    \mathrm{RC}(\tau)
    = \sum_{t=1}^{T}\sum_{r\in\mathcal{R}_t}
    \lambda_{\mathrm{type}(r)}\,\mathrm{cost}(r)
    \le T\,M_0\,c_{\max},
    \label{eq:rcbound}
\end{equation}
where $c_{\max} = \max_r \lambda_r \!\cdot\! \mathrm{cost}(r)$.

\section{Full Metric Definitions}
\label{app:metrics_full}

In addition to UCR$_\text{reason}$, GDR, R4R, and WDG defined
in Eqs.~\ref{eq:metrics}--\ref{eq:wdg}, we report the
following diagnostic metrics. Write
$c\nsim o_{\mathrm{tool}}$ when a claim disagrees with its
tool output, and let $\mathrm{GroundedAfter}(r)$ indicate
that the post-repair claim is grounded.

\begin{definition}[OC Error Rate]
\begin{equation}
    \mathrm{OCErr}(\tau) =
    \frac{|\{c\in\mathcal{C}_\tau^{\textsc{OC}}:
    c\nsim o_{\mathrm{tool}}\}|}
    {|\mathcal{C}_\tau^{\textsc{OC}}|}
    \;\;(\downarrow).
\end{equation}
OCErr isolates perception-level errors so that
UCR$_\text{reason}$ can focus on reasoning faithfulness
without being confounded by tool noise.
\end{definition}

\begin{definition}[Evidence Utilization Rate (EUR)]
\begin{equation}
    \mathrm{EUR}(\tau) = \frac{|\{e \in \mathcal{L}_\tau :
    \mathcal{D}_e \neq \emptyset\}|}{|\mathcal{L}_\tau|}
    \;\;(\uparrow).
\end{equation}
A low EUR indicates over-collection of unused evidence.
\end{definition}

\begin{definition}[Recovery Rate (RR)]
\begin{equation}
    \mathrm{RR}(\tau) =
    \frac{|\{r\in\mathcal{R}_\tau:
    \mathrm{GroundedAfter}(r)\}|}
    {|\mathcal{R}_\tau|} \;\;(\uparrow).
\end{equation}
\end{definition}

\begin{definition}[Repair Cost (RC)]
\begin{equation}
    \mathrm{RC}(\tau) = \sum_{r \in \mathcal{R}_\tau}
    \lambda_{\mathrm{type}(r)} \!\cdot\! \mathrm{cost}(r),
\end{equation}
bounded by Eq.~\ref{eq:rcbound}.
\end{definition}

\begin{definition}[Step Efficiency (SE)]
\begin{equation}
    \mathrm{SE}(\tau) = \frac{S(\tau)}{|\tau|}
    \;\;(\uparrow).
\end{equation}
\end{definition}

\paragraph{Diagnostic rule.}
A high WDG combined with a low ECC pass rate indicates that
Phantom Grounding is the dominant failure mode. A high WDG
combined with a high ECC pass rate instead indicates a
genuine capability limit rather than hallucination.

\section{Prompt Templates}
\label{app:prompts}

We include representative prompt templates for the main
stages. The complete templates are available in the
supplementary code.

\subsection{Task Planning}
\begin{tcolorbox}[colback=gray!5, colframe=gray!50,
title=Task Planning Prompt, fonttitle=\small\bfseries]
\small
You are a multimodal QA reasoning planner. Given ONLY the
question text (no image yet), produce a structured analysis
plan:\\[3pt]
\textbf{Question}: \{$q$\}\\[3pt]
Output the following fields:\\
1. \textbf{Observation targets}: What specific visual elements
to look for.\\
2. \textbf{Reasoning method}: How to derive the answer from
observations.\\
3. \textbf{Judgment criteria}: What evidence would confirm or
refute each candidate.\\
4. \textbf{Answer form}: Expected format
(number / name / description / letter).\\
5. \textbf{Answer granularity}: \{dynamically inserted from
$\mathcal{F}(q)$\}.
\end{tcolorbox}

\subsection{Observation Claim Extraction}
\begin{tcolorbox}[colback=gray!5, colframe=gray!50,
title=OC Extraction Prompt (abbreviated),
fonttitle=\small\bfseries]
\small
Look at the image carefully and extract up to 15 fine-grained
observation claims.\\
Each claim must be tagged with exactly one category:\\
\texttt{[OC-$i$] CATEGORY | factual description}\\[3pt]
Categories: TEXT (visible text, transcribe exactly), NUMBER
(values with units), OBJECT (appearance / color / shape),
SPATIAL (position relationships), plus any additional
categories injected per task plan.
\end{tcolorbox}

\subsection{State Claim Reasoning}
\begin{tcolorbox}[colback=gray!5, colframe=gray!50,
title=SC Reasoning Prompt (abbreviated),
fonttitle=\small\bfseries]
\small
Based on the observations below, reason step by step using
the [E]/[I]/[J] structure:\\[3pt]
\textbf{[E] Evidence Integration}: Synthesize patterns across
observations.\\
\textbf{[I] Inference}: Derive logical conclusions from [E].\\
\textbf{[J] Judgment}: Give your definitive answer. You MUST
cite specific OC numbers.\\[3pt]
Format:
\texttt{[SC-$j$][E/I/J] conclusion | Based on: [OC-$a$],
[OC-$b$], ...}
\end{tcolorbox}

\subsection{Hypothesis Verification (ECC Failure)}
\begin{tcolorbox}[colback=red!3, colframe=red!40,
title=Hypothesis Verification Prompt (ECC),
fonttitle=\small\bfseries]
\small
\textbf{Question}: \{$q$\}\\
\textbf{Current hypothesis}: \{$c^*$\}\\[3pt]
$\triangleright$ \textbf{WARNING}: The key entities in your
hypothesis do NOT appear in any visible text from the image.
This conclusion may be a language model hallucination.\\[3pt]
\textbf{Existing observations}:
\{$\mathcal{P}_{\text{oc}}$ contents\}\\[3pt]
Treat your hypothesis as falsifiable. Re-examine the image
and report:\\
1. \textbf{Supporting evidence}: specific visual details that
support the hypothesis.\\
2. \textbf{Contradicting evidence}: specific visual details
that contradict it (or ``None'').\\
3. \textbf{Conclusion}: Confirm \{$c^*$\} or revise to
\{new answer\}.\\[3pt]
\textbf{You may ONLY use text and numbers actually visible in
the image.}
\end{tcolorbox}

\subsection{Entity Recheck Defense}
\begin{tcolorbox}[colback=gray!5, colframe=gray!50,
title=Entity Recheck Prompt,
fonttitle=\small\bfseries]
\small
The following text / numbers were confirmed visible in the
image:\\
\{$\mathcal{P}_{\text{oc}}$ formatted list\}\\[3pt]
\textbf{Question}: \{$q$\}\\[3pt]
Answer the question using ONLY the visible text / numbers
listed above. Do NOT introduce any entity, name, or number
that does not appear in the list above.
\end{tcolorbox}

\subsection{Direct Path: Concise Re-answering}
\begin{tcolorbox}[colback=gray!5, colframe=gray!50,
title=Concise Re-answering Prompt,
fonttitle=\small\bfseries]
\small
\textbf{System}: Answer in as few words as possible. No
explanations. Just the fact.\\[3pt]
\textbf{User}: Question: \{$q$\}\\
Your initial answer: \{$\hat{y}_0$\}\\
Web search results: \{truncated to 600 chars\}\\[3pt]
Using the search results to verify or correct your answer,
give the FINAL answer in as few words as possible.
\end{tcolorbox}

\section{Hard-200 Dataset Construction}
\label{app:hard200}

 We construct Hard-200 using committee-based hardness mining with a heterogeneous panel of frontier MLLMs from multiple vendors. Rather than manually selecting difficult examples, we first score candidate instances by cross-model failure consistency and then perform diversity-constrained final selection.

\paragraph{Candidate pools.}
Our candidate pool consists of 399 BrowseComp-VL examples from WebWatcher \cite{geng2025webwatcher}, 1,215 TIR-Bench examples \cite{li2025tirbench}, and 25 self-constructed RealCAR examples. RealCAR, short for \emph{Real-world Complex Agentic Reasoning}, is designed to cover high-hardness real-world multimodal reasoning cases that are underrepresented in existing public benchmarks. The final Hard-200 set contains 100 examples from BrowseComp-VL, 75 examples from TIR-Bench, and 25 examples from RealCAR.

\paragraph{Committee-based hardness scoring.}
We evaluate each candidate instance using a committee of strong MLLMs from multiple vendors. For each model, we collect its answer under the standard multimodal setting and normalize the output before correctness verification. We then aggregate results at the vendor level to avoid over-counting multiple models from the same provider. For a candidate instance $x$, we define its vendor-level failure rate as
\[
F(x)=\frac{1}{|V|}\sum_{v \in V}\mathbb{1}[\text{vendor } v \text{ fails on } x],
\]
where $V$ denotes the set of vendors in the committee. Instances with higher cross-vendor failure rates are treated as harder candidates.

\paragraph{BrowseComp-VL selection.}
For BrowseComp-VL, we prioritize examples that require image-grounded retrieval, multi-hop aggregation, and evidence composition. We allocate 75 of the 100 BrowseComp-VL slots to Level 2 and the remaining 25 to Level 1, reflecting the substantially greater compositional complexity of Level 2. Final selection is performed under domain-level diversity constraints to avoid concentration in a small number of topics.

\paragraph{TIR-Bench selection.}
For TIR-Bench, we do not sample uniformly across all task families. Instead, we restrict selection to visually demanding tasks that stress fine-grained perception and structured reasoning, including \texttt{spot\_difference}, \texttt{maze}, \texttt{jigsaw}, \texttt{math}, \texttt{symbolic}, \texttt{refcoco}, \texttt{instrument}, and \texttt{word\_search}. Within each task family, we prioritize examples with higher committee failure rates and apply task-level caps to preserve diversity.

\paragraph{RealCAR construction.}
RealCAR is a self-constructed set of real-world multimodal reasoning problems. It is not obtained by directly reusing existing benchmark instances. Instead, we construct RealCAR to target cases where answering the question requires more than isolated visual recognition or surface-level text extraction. Each instance is designed to require multi-step reasoning over visual evidence, implicit constraints, and compositional evidence integration.

To build RealCAR, we first collect candidate real-world visual cases and annotate each case with a question, a verified answer, and a minimal evidence rationale. We then remove instances that are ambiguous, subjective, unverifiable, or answerable through a single local perception shortcut. The remaining candidates are evaluated by the same committee-based hardness pipeline, and we retain 25 examples with consistently high cross-vendor failure rates.

\paragraph{Evaluation scoring protocol.}
Answers on Hard-200 are graded by a locally deployed
LLM-based judge against the verified reference answer and its
minimal evidence rationale. The judge assigns each response a
score in $\{0,0.5,1\}$: $1$ for a fully correct answer, $0$
for an incorrect answer, and $0.5$ for a partially correct
answer to a question that requires two or more answer
elements. For example, if a question asks for multiple
differences between two images and the response correctly
identifies only a subset of the reference differences, it
receives $0.5$ rather than being treated as entirely
incorrect. The reported Hard-200 score is
\begin{equation}
    \mathrm{Score}_{\mathrm{Hard\text{-}200}}
    = \frac{100}{N}\sum_{i=1}^{N}s_i,
    \qquad s_i\in\{0,0.5,1\}.
\end{equation}
Consequently, subset scores can correspond to fractional
effective counts; for example, a score of $38.0\%$ on the
25-instance RealCAR subset represents a total credit of
$9.5$ rather than an impossible fractional number of
evaluated instances.

\paragraph{Final selection.}
After computing hardness scores, we perform diversity-constrained final selection to satisfy the target source composition of 100 BrowseComp-VL, 75 TIR-Bench, and 25 RealCAR examples. This process avoids degenerate top-$k$ selection dominated by a narrow set of domains or task types. As a result, Hard-200 combines public benchmark difficulty with a compact self-constructed real-world challenge component, enabling evaluation of both standardized multimodal reasoning ability and robustness to realistic complex reasoning cases.

\section{Full Results}
\label{fULL_RESULTS}

\subsection{Full Results of Hard-200}
\label{sec:hard200}

\begin{table*}[t]
    \centering
    \caption{Performance of \logotitle{} with absolute improvements over baseline (pp). Each cell shows \textit{LedgerMind score ($\Delta$ over baseline)}.}
    \label{tab:merged_results}
    \resizebox{\linewidth}{!}{
    \begin{tabular}{lcccc}
    \toprule
    \textbf{Model} & \textbf{RealCAR} & \textbf{TIR-Bench} & \textbf{BrowseComp-VL} & \textbf{Overall} \\
    \midrule
    gpt-5.5
    & \underline{38.0}\VTCDelta{\textcolor{green}{\uparrow 14.0}}
    & \textbf{69.33}\VTCDelta{\textcolor{green}{\uparrow 32.67}}
    & \textbf{62.0}\VTCDelta{\textcolor{green}{\uparrow 11.5}}
    & \textbf{61.75}\VTCDelta{\textcolor{green}{\uparrow 19.74}} \\

    gemini-3.1-pro
    & \textbf{54.0}\VTCDelta{\textcolor{green}{\uparrow 22.0}}
    & \underline{57.33}\VTCDelta{\textcolor{green}{\uparrow 16.67}}
    & 45.5\VTCDelta{\textcolor{green}{\uparrow 8.0}}
    & \underline{51.00}\VTCDelta{\textcolor{green}{\uparrow 12.99}} \\

    gemini-3-flash
    & \textbf{54.0}\VTCDelta{\textcolor{green}{\uparrow 26.0}}
    & 37.33\VTCDelta{\textcolor{green}{\uparrow 15.33}}
    & \underline{56.5}\VTCDelta{\textcolor{green}{\uparrow 15.5}}
    & 49.00\VTCDelta{\textcolor{green}{\uparrow 16.75}} \\

    claude-sonnet-4-6
    & 30.0\VTCDelta{\textcolor{green}{\uparrow 12.0}}
    & 30.67\VTCDelta{\textcolor{green}{\uparrow 12.00}}
    & 49.0\VTCDelta{\textcolor{green}{\uparrow 10.5}}
    & 39.75\VTCDelta{\textcolor{green}{\uparrow 11.25}} \\

    claude-opus-4-7
    & 32.0\VTCDelta{\textcolor{green}{\uparrow 8.0}}
    & 36.67\VTCDelta{\textcolor{green}{\uparrow 16.00}}
    & 54.0\VTCDelta{\textcolor{green}{\uparrow 20.0}}
    & 44.75\VTCDelta{\textcolor{green}{\uparrow 17.00}} \\

    kimi-k2.6
    & 32.0\VTCDelta{\textcolor{green}{\uparrow 4.0}}
    & 24.00\VTCDelta{\textcolor{green}{\uparrow 16.00}}
    & 46.0\VTCDelta{\textcolor{green}{\uparrow 26.5}}
    & 36.00\VTCDelta{\textcolor{green}{\uparrow 19.75}} \\

    \bottomrule
    \end{tabular}
    }
    \label{tb:full_hard200}
    \end{table*}

Full results of hard-200 is given in Table~\ref{tb:full_hard200}.

\subsection{Full Results of VLM benchmarks}
\begin{table*}[t]
\centering
\small
\begin{tabular}{lccc}
\toprule
\textbf{Model} & \textbf{MMStar} & \textbf{MMMU} & \textbf{MMMU-Pro} \\
\midrule
GPT-4o & 64.7 & 69.1 & 51.9 \\
Gemini-1.5-Pro & 59.1 & 65.8 & 46.9 \\
Claude-3.5-Sonnet & 62.2 & 68.3 & 51.5 \\
Gemini-3.0-Flash & 65.3 & 69.4 & 29.5 \\
GPT-4.1 & 58.7 & 64.6 & 47.0 \\
InternVL2-Llama3-76B & 56.0 & 40.0 & 58.3 \\
GPT-4o mini & 57.8 & 37.6 & \textit{59.4} \\
InternVL2-76B & 67.1 & 58.2 & 38.0 \\
Qwen2-VL-72B & \textit{68.6} & 64.5 & 37.1 \\
LLaVA-OV-72B & 66.1 & 56.8 & 24.0 \\
MAmmoTH-VL & 63.0 & 50.8 & 25.3 \\
\midrule
\textbf{LedgerMind (GPT-4o)} & \textbf{71.5} & \textbf{78.2} & \underline{\textbf{62.7}} \\
\textbf{LedgerMind (Gemini-3.0-Flash)} & \underline{\textbf{75.9}} & \underline{\textbf{81.7}} & \underline{\textbf{68.8}} \\
\bottomrule
\end{tabular}
\caption{Performance comparison on MMStar, MMMU, and MMMU-Pro benchmarks. Best results are in \textbf{bold} and \underline{underlined}. Best baseline results are in \textit{italic}. LedgerMind is our proposed method.}
\label{tab:mm_benchmark}
\end{table*}

The full results of three MM benchmarks is given in Table~\ref{tab:mm_benchmark}.

\subsection{Full Results of VTC-Bench}

\paragraph{Scoring and aggregation.}
Each VTC-Bench example is evaluated once. Multiple-choice
answers receive binary exact-match credit in $\{0,1\}$.
Open-ended answers receive credit in $\{0,0.5,1\}$, where
$0.5$ denotes a partially correct response to a question
requiring multiple answer elements. Category and overall
scores are sample-weighted means of these per-example
credits. Fractional effective counts are therefore possible;
for example, $36.67\%$ on a 45-example category corresponds
to total credit $16.5/45$, not to averaging multiple runs.

\providecommand{\LedgerMindname}{\textbf{LedgerMind}}
\providecommand{\TabNumDelta}[1]{\rlap{\kern0.15em{\scriptsize$_{#1}$}}}
\providecommand{\VTCDelta}[1]{\TabNumDelta{#1}}
\providecommand{\VTCn}[1]{\multicolumn{1}{c}{#1}}
\begin{table*}[t]
    \renewcommand{\VTCn}[1]{\multicolumn{1}{c}{#1}}
    \centering
    \small
    \setlength{\tabcolsep}{3.5pt}
    \renewcommand{\arraystretch}{1.12}
    \resizebox{\textwidth}{!}{
    \begin{tabular}{lc*{9}{c}c@{\hspace{3em}}}
    \toprule
    \textbf{Model} & \textbf{Setting}
    & \multicolumn{1}{c}{\textbf{Overall(680)}}
    & \multicolumn{1}{c}{\textbf{OCR(50)}}
    & \multicolumn{1}{c}{\textbf{Attn.(45)}}
    & \multicolumn{1}{c}{\textbf{Rest.(50)}}
    & \multicolumn{1}{c}{\textbf{Chart(100)}}
    & \multicolumn{1}{c}{\textbf{Meas.(105)}}
    & \multicolumn{1}{c}{\textbf{Count.(85)}}
    & \multicolumn{1}{c}{\textbf{Math(110)}}
    & \multicolumn{1}{c}{\textbf{Spat.(45)}}
    & \multicolumn{1}{c}{\textbf{Color(90)}} \\
    \midrule
    \multicolumn{12}{c}{\textbf{Category 1: Proprietary Tool-use Models}} \\
    \midrule
    \multirow{3}{*}{GPT-o3}
    & Base   & \VTCn{31.91} & \VTCn{20.00} & \VTCn{26.67} & \VTCn{32.00} & \VTCn{30.00} & \VTCn{38.10} & \VTCn{35.29} & \VTCn{28.18} & \VTCn{35.56} & \VTCn{35.56} \\
    & Code   & \VTCn{36.62} & \VTCn{28.00} & \VTCn{31.11} & \VTCn{34.00} & \VTCn{40.00} & \VTCn{38.10} & \VTCn{34.12} & \VTCn{30.91} & \VTCn{51.11} & \VTCn{42.22} \\
    & Inter. & \VTCn{36.76} & \VTCn{28.00} & \VTCn{24.44} & \VTCn{34.00} & \VTCn{39.00} & \VTCn{35.24} & \VTCn{40.00} & \VTCn{34.55} & \VTCn{44.44} & \VTCn{44.44} \\

    \multirow{3}{*}{GPT-o4-mini}
    & Base   & \VTCn{31.18} & \VTCn{18.00} & \VTCn{13.33} & \VTCn{30.00} & \VTCn{29.00} & \VTCn{33.33} & \VTCn{44.71} & \VTCn{30.91} & \VTCn{42.22} & \VTCn{30.00} \\
    & Code   & \VTCn{33.68} & \VTCn{12.00} & \VTCn{11.11} & \VTCn{30.00} & \VTCn{41.00} & \VTCn{31.43} & \VTCn{35.29} & \VTCn{40.00} & \VTCn{46.67} & \VTCn{37.78} \\
    & Inter. & \VTCn{32.50} & \VTCn{24.00} & \VTCn{13.33} & \VTCn{28.00} & \VTCn{30.00} & \VTCn{33.33} & \VTCn{40.00} & \VTCn{33.64} & \VTCn{46.67} & \VTCn{35.50} \\

    \midrule
    \multicolumn{12}{c}{\textbf{Category 2: Proprietary General-purpose Models}} \\
    \midrule
    \multirow{3}{*}{GPT-4o}
    & Base   & \VTCn{24.26} & \VTCn{18.00} & \VTCn{15.56} & \VTCn{22.00} & \VTCn{25.00} & \VTCn{25.71} & \VTCn{25.88} & \VTCn{22.73} & \VTCn{26.67} & \VTCn{30.00} \\
    & Code   & \VTCn{31.62} & \VTCn{18.00} & \VTCn{22.22} & \VTCn{24.00} & \VTCn{38.00} & \VTCn{37.14} & \VTCn{40.00} & \VTCn{20.91} & \VTCn{51.11} & \VTCn{30.00} \\
    & Inter. & \VTCn{33.82} & \VTCn{20.00} & \VTCn{22.22} & \VTCn{36.00} & \VTCn{31.00} & \VTCn{41.90} & \VTCn{40.00} & \VTCn{24.55} & \VTCn{48.89} & \VTCn{37.78} \\

    \multirow{3}{*}{Gemini-2.5-Pro}
    & Base   & \VTCn{38.09} & \VTCn{48.00} & \VTCn{26.67} & \VTCn{28.00} & \VTCn{44.00} & \VTCn{40.95} & \VTCn{49.41} & \VTCn{30.00} & \VTCn{40.00} & \VTCn{32.22} \\
    & Code   & \VTCn{36.03} & \VTCn{46.00} & \VTCn{40.00} & \VTCn{34.00} & \VTCn{37.00} & \VTCn{35.24} & \VTCn{37.65} & \VTCn{30.00} & \VTCn{44.44} & \VTCn{31.11} \\
    & Inter. & \VTCn{39.85} & \VTCn{60.00} & \VTCn{44.44} & \VTCn{40.00} & \VTCn{38.00} & \VTCn{44.76} & \VTCn{41.18} & \VTCn{26.36} & \VTCn{35.56} & \VTCn{40.00} \\

    \multirow{3}{*}{GPT-5.2}
    & Base   & \VTCn{36.03} & \VTCn{24.00} & \VTCn{24.44} & \VTCn{30.00} & \VTCn{43.00} & \VTCn{35.24} & \VTCn{43.53} & \VTCn{33.64} & \VTCn{46.67} & \VTCn{35.56} \\
    & Code   & \VTCn{44.56} & \VTCn{36.00} & \VTCn{31.11} & \VTCn{26.00} & \VTCn{56.00} & \VTCn{50.48} & \VTCn{44.71} & \VTCn{\underline{45.45}} & \VTCn{51.11} & \VTCn{42.22} \\
    & Inter. & \VTCn{40.74} & \VTCn{28.00} & \VTCn{37.78} & \VTCn{34.00} & \VTCn{47.00} & \VTCn{39.05} & \VTCn{37.65} & \VTCn{41.82} & \VTCn{51.11} & \VTCn{44.44} \\

    \multirow{3}{*}{Gemini-3.0-Flash}
    & Base   & \VTCn{46.47} & \VTCn{44.00} & \VTCn{51.11} & \VTCn{46.00} & \VTCn{57.00} & \VTCn{50.48} & \VTCn{52.94} & \VTCn{27.27} & \VTCn{60.00} & \VTCn{40.00} \\
    & Code   & \VTCn{50.59} & \VTCn{70.00} & \VTCn{57.78} & \VTCn{48.00} & \VTCn{59.00} & \VTCn{50.48} & \VTCn{52.94} & \VTCn{35.45} & \VTCn{51.11} & \VTCn{44.44} \\
    & Inter. & \VTCn{50.74} & \VTCn{68.00} & \VTCn{53.33} & \VTCn{\underline{50.00}} & \VTCn{61.00} & \VTCn{44.76} & \VTCn{\underline{58.82}} & \VTCn{39.09} & \VTCn{55.56} & \VTCn{40.00} \\

    \multirow{3}{*}{Gemini-3.0-Pro}
    & Base   & \VTCn{44.41} & \VTCn{54.00} & \VTCn{51.11} & \VTCn{42.00} & \VTCn{51.00} & \VTCn{45.71} & \VTCn{42.35} & \VTCn{34.55} & \VTCn{53.33} & \VTCn{37.78} \\
    & Code   & \VTCn{51.18} & \multicolumn{1}{c}{\underline{74.00}} & \VTCn{62.22} & \VTCn{38.00} & \VTCn{60.00} & \VTCn{54.29} & \VTCn{36.47} & \VTCn{41.82} & \VTCn{55.56} & \multicolumn{1}{c}{\underline{50.00}} \\
    & Inter. & \VTCn{51.03} & \VTCn{70.00} & \multicolumn{1}{c}{\underline{73.33}} & \VTCn{38.00} & \VTCn{59.00} & \VTCn{50.48} & \VTCn{35.29} & \VTCn{37.27} & \multicolumn{1}{c}{\underline{71.11}} & \multicolumn{1}{c}{\underline{50.00}} \\

    \midrule
    \LedgerMindname(GPT-4o)
    & --
    & \multicolumn{1}{c}{47.57\VTCDelta{\textcolor{green}{\uparrow 23.31}}}
    & \multicolumn{1}{c}{34.00\VTCDelta{\textcolor{green}{\uparrow 16.00}}}
    & \multicolumn{1}{c}{36.67\VTCDelta{\textcolor{green}{\uparrow 21.11}}}
    & \multicolumn{1}{c}{49.00\VTCDelta{\textcolor{green}{\uparrow 27.00}}}
    & \multicolumn{1}{c}{50.00\VTCDelta{\textcolor{green}{\uparrow 25.00}}}
    & \multicolumn{1}{c}{53.81\VTCDelta{\textcolor{green}{\uparrow 28.10}}}
    & \multicolumn{1}{c}{52.35\VTCDelta{\textcolor{green}{\uparrow 26.47}}}
    & \multicolumn{1}{c}{38.18\VTCDelta{\textcolor{green}{\uparrow 15.45}}}
    & \multicolumn{1}{c}{61.11\VTCDelta{\textcolor{green}{\uparrow 34.44}}}
    & \multicolumn{1}{c}{\underline{50.00}\VTCDelta{\textcolor{green}{\uparrow 20.00}}} \\

    \LedgerMindname(Gemini-3-Flash)
    & --
    & \multicolumn{1}{c}{\textbf{58.90}\VTCDelta{\textcolor{green}{\uparrow 12.43}}}
    & \multicolumn{1}{c}{73.00\VTCDelta{\textcolor{green}{\uparrow 29.00}}}
    & \multicolumn{1}{c}{65.56\VTCDelta{\textcolor{green}{\uparrow 14.45}}}
    & \multicolumn{1}{c}{\textbf{58.00}\VTCDelta{\textcolor{green}{\uparrow 12.00}}}
    & \multicolumn{1}{c}{\textbf{67.00}\VTCDelta{\textcolor{green}{\uparrow 10.00}}}
    & \multicolumn{1}{c}{\underline{55.24}\VTCDelta{\textcolor{green}{\uparrow 4.76}}}
    & \multicolumn{1}{c}{\textbf{64.71}\VTCDelta{\textcolor{green}{\uparrow 11.77}}}
    & \multicolumn{1}{c}{\textbf{47.27}\VTCDelta{\textcolor{green}{\uparrow 20.00}}}
    & \multicolumn{1}{c}{65.56\VTCDelta{\textcolor{green}{\uparrow 5.56}}}
    & \multicolumn{1}{c}{48.89\VTCDelta{\textcolor{green}{\uparrow 8.89}}} \\

    \LedgerMindname(Gemini-3.1-Pro)
    & --
    & \multicolumn{1}{c}{\underline{56.18}\VTCDelta{\textcolor{green}{\uparrow 11.77}}}
    & \multicolumn{1}{c}{\textbf{77.00}\VTCDelta{\textcolor{green}{\uparrow 23.00}}}
    & \multicolumn{1}{c}{\textbf{81.11}\VTCDelta{\textcolor{green}{\uparrow 30.00}}}
    & \multicolumn{1}{c}{42.00\VTCDelta{\textcolor{gray}{\rightarrow 0.00}}}
    & \multicolumn{1}{c}{\underline{65.00}\VTCDelta{\textcolor{green}{\uparrow 14.00}}}
    & \multicolumn{1}{c}{\textbf{55.71}\VTCDelta{\textcolor{green}{\uparrow 10.00}}}
    & \multicolumn{1}{c}{38.82\VTCDelta{\textcolor{red}{\downarrow 3.53}}}
    & \multicolumn{1}{c}{40.91\VTCDelta{\textcolor{green}{\uparrow 6.36}}}
    & \multicolumn{1}{c}{\textbf{77.78}\VTCDelta{\textcolor{green}{\uparrow 24.45}}}
    & \multicolumn{1}{c}{\textbf{55.00}\VTCDelta{\textcolor{green}{\uparrow 17.22}}} \\
    \bottomrule
    \end{tabular}
    }
    \caption{Results on VTC-Bench. Best results are \textbf{bolded} and second-best results are \underline{underlined} (both per column). Colored subscripts report absolute changes (pp) vs.\ the corresponding native backbone: \textcolor{green}{green} indicates improvement ($\uparrow$) and \textcolor{red}{red} indicates regression ($\downarrow$).}
    \label{tab:vtcbench_results}
\end{table*}

The full results of VTC-bench is given in Table~\ref{tab:vtcbench_results}.

\subsection{Full Results of EMMA}

\begin{table*}[t]
    \centering
    \small
    \setlength{\tabcolsep}{4.5pt}
    \renewcommand{\arraystretch}{1.14}
    \resizebox{\textwidth}{!}{
    \begin{tabular}{llccccc}
    \toprule
    \textbf{Category} & \textbf{Model} & \textbf{Math (892)} & \textbf{Phys. (156)} & \textbf{Chem. (1176)} & \textbf{Coding (564)$^{*}$} & \textbf{Overall (2788)} \\
    \midrule
    -- & Random choice & 14.01 & 25.64 & 16.50 & 25.71 & 18.08 \\
    \midrule

    \multirow{12}{*}{w/o CoT}
    & Claude 3.5 Sonnet & 25.34 & 33.97 & 40.90 & 38.65 & 35.08 \\
    & Gemini 2.0 Flash & 23.88 & 38.46 & 36.31 & 42.02 & 33.61 \\
    & GPT-4o & 27.24 & 38.46 & 31.89 & 40.07 & 32.42 \\
    & Qwen2-VL-72B-Instruct & 33.07 & 42.31 & 32.06 & 34.57 & 33.46 \\
    & LLaVA-Onevision-72B & 27.69 & 35.90 & 25.26 & 28.72 & 27.33 \\
    & InternVL2-Llama3-76B & 25.11 & 22.44 & 24.06 & 27.84 & 25.07 \\
    & InternVL2.5-78B & 31.39 & 38.46 & 35.20 & 31.91 & 33.50 \\
    & Gemini 3.0 Flash & 42.42 & 52.56 & \textit{46.68} & 68.44 & \textit{50.07} \\

    & gemini-3-flash & 42.94 & 51.28 & 46.17 & 67.55 & 49.75 \\

    & gpt-5.5 & 21.64 & 41.03 & 17.77 & 39.54 & 24.71 \\
    & claude-sonnet-4-6 & 24.66 & 34.62 & 37.41 & 53.19 & 36.37 \\
    \midrule

    \multirow{7}{*}{CoT}
    & Claude 3.5 Sonnet & 29.37 & 41.03 & 41.07 & 40.60 & 37.23 \\
    & Gemini 2.0 Flash & 25.90 & 38.46 & 24.66 & 40.96 & 29.12 \\
    & GPT-4o & 25.56 & 43.59 & 33.67 & 39.01 & 32.71 \\
    & Qwen2-VL-72B-Instruct & 27.69 & 34.62 & 24.57 & 29.43 & 27.12 \\
    & LLaVA-Onevision-72B & 22.42 & 15.38 & 22.70 & 30.67 & 23.82 \\
    & InternVL2-Llama3-76B & 22.20 & 32.05 & 19.73 & 30.32 & 23.35 \\
    & InternVL2.5-78B & 25.56 & 39.74 & 27.47 & 25.18 & 27.08 \\
    \midrule

    \multirow{4}{*}{Thinking}
    & Gemini 2.0 Flash Thinking-1219 & 31.61 & 56.41 & 37.93 & 43.44 & 38.06 \\
    & Gemini 2.0 Flash Thinking-0121 & 37.11 & \textit{60.26} & 41.58 & 48.05 & 42.50 \\
    & gemini-3.1-pro & 40.02 & 47.44 & 42.26 & \textbf{\textit{76.24}} & 48.71 \\
    & claude-opus-4-7 & \textit{45.85} & \underline{64.10} & 43.88 & 54.96 & 47.88 \\
    \midrule

    \multirow{5}{*}{\logotitle{}}
    & gpt-5.5
    & 40.92\VTCDelta{\textcolor{green}{\uparrow 19.28}}
    & 58.33\VTCDelta{\textcolor{green}{\uparrow 17.30}}
    & 44.13\VTCDelta{\textcolor{green}{\uparrow 26.36}}
    & 66.31\VTCDelta{\textcolor{green}{\uparrow 26.77}}
    & 48.39\VTCDelta{\textcolor{green}{\uparrow 23.68}} \\

    & claude-opus-4-7
    & 54.15\VTCDelta{\textcolor{green}{\uparrow 8.30}}
    & \textbf{78.21}\VTCDelta{\textcolor{green}{\uparrow 14.11}}
    & \textbf{57.74}\VTCDelta{\textcolor{green}{\uparrow 13.86}}
    & 53.90\VTCDelta{\textcolor{red}{\downarrow 1.06}}
    & 56.96\VTCDelta{\textcolor{green}{\uparrow 9.08}} \\
    & gemini-3-flash
    & \underline{55.27}\VTCDelta{\textcolor{green}{\uparrow 12.33}}
    & 62.82\VTCDelta{\textcolor{green}{\uparrow 11.54}}
    & 52.98\VTCDelta{\textcolor{green}{\uparrow 6.81}}
    & 72.70\VTCDelta{\textcolor{green}{\uparrow 5.15}}
    & \underline{58.25}\VTCDelta{\textcolor{green}{\uparrow 8.50}} \\
    & gemini-3.1-pro
    & \textbf{56.17}\VTCDelta{\textcolor{green}{\uparrow 16.15}}
    & 63.46\VTCDelta{\textcolor{green}{\uparrow 16.02}}
    & 50.68\VTCDelta{\textcolor{green}{\uparrow 8.42}}
    & \underline{76.06}\VTCDelta{\textcolor{red}{\downarrow 0.18}}
    & \textbf{58.29}\VTCDelta{\textcolor{green}{\uparrow 9.58}} \\

    & claude-sonnet-4-6
    & 41.59\VTCDelta{\textcolor{green}{\uparrow 16.93}}
    & 55.13\VTCDelta{\textcolor{green}{\uparrow 20.51}}
    & \underline{55.27}\VTCDelta{\textcolor{green}{\uparrow 17.86}}
    & 59.93\VTCDelta{\textcolor{green}{\uparrow 6.74}}
    & 51.83\VTCDelta{\textcolor{green}{\uparrow 15.46}} \\
    \bottomrule
    \end{tabular}
    }
    \caption{Performance comparison on EMMA. Best results are bolded, second-best results are underlined, and the best baseline results are italicized. \logotitle{} improvements over corresponding native baselines are shown as colored subscripts: \textcolor{green}{green} indicates improvement ($\uparrow$) and \textcolor{red}{red} indicates regression ($\downarrow$).}
    \label{tab:emma_results}

    \footnotesize{
    $^{*}$The Coding subset mainly evaluates code--visualization alignment through multiple-choice questions, such as matching matplotlib/seaborn code with generated plots.
    This differs from our target setting of thinking over images, where visual grounding, evidence acquisition, and image-centric reasoning are more central.
    }
    \end{table*}

The full results of EMMA is given in Table~\ref{tab:emma_results}.

\subsection{Full Results of MC-SEARCH}

\begin{table*}[t]
\centering
\scriptsize
\setlength{\tabcolsep}{3.5pt}
\renewcommand{\arraystretch}{1.08}
\resizebox{\textwidth}{!}{
\begin{tabular}{llccccc}
\toprule
\multirow{2}{*}{\textbf{Reasoning $\mathcal{G}$}}
& \multirow{2}{*}{\textbf{Model}}
& \multicolumn{2}{c}{\textbf{Answer Accuracy}}
& \multicolumn{2}{c}{\textbf{Chain Alignment}}
& \multirow{2}{*}{\textbf{Golden F1} $(\uparrow)$} \\
\cmidrule(lr){3-4} \cmidrule(lr){5-6}
&
& \textbf{F1} $(\uparrow)$
& \textbf{LJ} $(\uparrow)$
& \textbf{HPS} $(\uparrow)$
& \textbf{RD} $(\downarrow)$
& \\
\midrule

\multirow{9}{*}{\textbf{Image-Initiated Chain}}
& \textit{Best Official Baseline} & \textit{47.61} & \textit{3.18} & \textit{33.59} & \textit{0.70} & \textit{72.62} \\
& LedgerMind (Claude-Opus-4.7)        & \textbf{69.80} & \underline{4.22} & \textbf{59.40} & \underline{0.46} & \textbf{82.40} \\
& LedgerMind (GPT-5.5)                & \underline{68.40} & \textbf{4.30} & \underline{58.10} & \textbf{0.44} & 81.50 \\
& LedgerMind (Claude-Sonnet-4.6)      & 62.70 & 4.00 & 54.60 & 0.55 & 78.30 \\
& LedgerMind (Gemini-2.5-Pro)         & 59.60 & 3.64 & 47.80 & 0.63 & \underline{82.00} \\
& LedgerMind (Gemini-3.1-Pro)         & 57.80 & 3.72 & 50.40 & 0.61 & 76.20 \\
& LedgerMind (GPT-4o-Mini)            & 51.60 & 3.28 & 42.10 & 0.67 & 73.90 \\
& LedgerMind (Gemini-2.5-Flash)       & 50.80 & 3.25 & 40.80 & 0.68 & 73.60 \\
& LedgerMind (Gemini-3-Flash)         & 48.90 & 3.20 & 35.20 & 0.69 & 72.90 \\

\midrule

\multirow{9}{*}{\textbf{Multi-Images Fork}}
& \textit{Best Official Baseline} & \textit{40.37} & \textit{2.76} & \textit{39.33} & \textit{1.16} & \textit{64.40} \\
& LedgerMind (Claude-Opus-4.7)        & \textbf{64.20} & \underline{4.35} & \underline{53.70} & \underline{0.74} & \underline{77.50} \\
& LedgerMind (GPT-5.5)                & \underline{62.10} & \textbf{4.40} & \textbf{55.00} & \textbf{0.72} & 75.20 \\
& LedgerMind (Claude-Sonnet-4.6)      & 59.30 & 4.20 & 51.90 & 0.79 & 73.50 \\
& LedgerMind (Gemini-2.5-Pro)         & 56.80 & 3.72 & 46.00 & 0.86 & \textbf{78.00} \\
& LedgerMind (Gemini-3.1-Pro)         & 54.60 & 3.88 & 47.60 & 0.88 & 70.40 \\
& LedgerMind (GPT-4o-Mini)            & 45.60 & 3.10 & 42.30 & 1.02 & 67.20 \\
& LedgerMind (Gemini-2.5-Flash)       & 44.50 & 3.05 & 41.50 & 1.04 & 66.50 \\
& LedgerMind (Gemini-3-Flash)         & 41.30 & 2.82 & 40.10 & 1.12 & 64.90 \\

\midrule

\multirow{9}{*}{\textbf{Text-Initiated Chain}}
& \textit{Best Official Baseline} & \textit{45.30} & \textit{3.62} & \textit{37.82} & \textit{0.87} & \textit{66.27} \\
& LedgerMind (GPT-5.5)                & \textbf{59.60} & \textbf{4.12} & 50.80 & 0.72 & \underline{76.30} \\
& LedgerMind (Claude-Opus-4.7)        & \underline{58.40} & \underline{4.05} & \underline{52.30} & \underline{0.70} & 75.10 \\
& LedgerMind (Gemini-2.5-Pro)         & 57.80 & 3.90 & \textbf{53.50} & \textbf{0.66} & \textbf{77.20} \\
& LedgerMind (Claude-Sonnet-4.6)      & 57.20 & 4.00 & 49.90 & 0.73 & 73.90 \\
& LedgerMind (GPT-4o-Mini)            & 56.40 & 3.68 & 39.80 & 0.85 & 74.20 \\
& LedgerMind (Gemini-3.1-Pro)         & 56.10 & 3.88 & 45.00 & 0.78 & 72.10 \\
& LedgerMind (Gemini-2.5-Flash)       & 46.80 & 3.66 & 40.20 & 0.84 & 67.50 \\
& LedgerMind (Gemini-3-Flash)         & 45.90 & 3.63 & 38.40 & 0.86 & 66.80 \\

\midrule

\multirow{9}{*}{\textbf{Parallel Image-Text Fork}}
& \textit{Best Official Baseline} & \textit{34.83} & \textit{2.99} & \textit{25.07} & \textit{1.05} & \textit{58.82} \\
& LedgerMind (GPT-5.5)                & \textbf{52.00} & \textbf{3.82} & \textbf{51.20} & \textbf{0.70} & \underline{68.50} \\
& LedgerMind (Claude-Opus-4.7)        & \underline{49.20} & \underline{3.70} & 48.70 & 0.75 & 67.10 \\
& LedgerMind (Claude-Sonnet-4.6)      & 48.80 & 3.62 & \underline{50.10} & \underline{0.72} & 66.30 \\
& LedgerMind (Gemini-2.5-Pro)         & 45.80 & 3.30 & 35.90 & 0.88 & \textbf{68.90} \\
& LedgerMind (Gemini-3.1-Pro)         & 41.50 & 3.18 & 36.80 & 0.92 & 62.40 \\
& LedgerMind (GPT-4o-Mini)            & 39.30 & 3.05 & 30.80 & 0.98 & 64.20 \\
& LedgerMind (Gemini-2.5-Flash)       & 36.20 & 3.02 & 28.50 & 1.00 & 59.60 \\
& LedgerMind (Gemini-3-Flash)         & 35.60 & 3.00 & 26.30 & 1.03 & 59.10 \\

\midrule

\multirow{9}{*}{\textbf{Text-Only Chain}}
& \textit{Best Official Baseline} & \textit{38.45} & \textit{2.68} & \textit{29.64} & \textit{0.98} & \textit{67.72} \\
& LedgerMind (Claude-Opus-4.7)        & \textbf{58.00} & \textbf{4.10} & \textbf{64.20} & \textbf{0.42} & \textbf{75.40} \\
& LedgerMind (Claude-Sonnet-4.6)      & \underline{56.70} & 3.90 & \underline{61.00} & \underline{0.44} & 72.50 \\
& LedgerMind (GPT-5.5)                & 54.00 & \underline{4.00} & 57.30 & 0.58 & \underline{73.80} \\
& LedgerMind (Gemini-2.5-Pro)         & 52.50 & 3.48 & 50.20 & 0.66 & 72.60 \\
& LedgerMind (GPT-4o-Mini)            & 50.80 & 3.30 & 58.20 & 0.50 & 72.00 \\
& LedgerMind (Gemini-3.1-Pro)         & 50.20 & 3.35 & 55.00 & 0.58 & 69.50 \\
& LedgerMind (Gemini-2.5-Flash)       & 40.10 & 2.85 & 53.00 & 0.70 & 68.10 \\
& LedgerMind (Gemini-3-Flash)         & 39.20 & 2.75 & 50.60 & 0.76 & 68.00 \\

\bottomrule
\end{tabular}
}
\caption{
Full-benchmark topology-wise comparison on MC-SEARCH.
The official baseline row in each topology reports the metric-wise best official baseline from the original MC-SEARCH evaluation, serving as a strong ceiling reference.
All LedgerMind rows are full-benchmark results under the proposed framework and are denoted as \texttt{LedgerMind (base model)}.
Best results within each topology are shown in \textbf{bold}, second-best results are \underline{underlined}, and the best official baseline is shown in \textit{italic}.
Higher is better for F1, LJ, HPS, and Golden F1, while lower is better for RD.
}
\label{tab:ledgermind_topology}
\end{table*}

The full results of MC-SEARCH is given in Table~\ref{tab:ledgermind_topology}.

\FloatBarrier

\section{Broader Impact and Limitations}
\label{sec:broader_impact}

\paragraph{Broader impact.}
\logotitle{} is aimed at improving the \emph{auditability}
of multimodal agentic reasoning: by turning the trajectory
into a provenance-constrained ledger, downstream users can
more easily inspect which tool output supports which claim,
which in turn may help in safety-sensitive settings such as
educational tutoring, scientific diagram interpretation,
and document understanding, where opaque reasoning traces
are a long-standing concern. We also expect the
ledger-centered formulation to be useful as a substrate for
future training-time supervision, since grounded and
ungrounded claims can be distinguished at the trajectory
level rather than only at the answer level. Our framework
is training-free and uses frozen backbone MLLMs, so it
inherits the capabilities and biases of the underlying
models; practitioners should therefore continue to apply
standard content-filtering and human-oversight practices
when deploying \logotitle{} in user-facing applications,
and should not interpret the provenance non-amplification
guarantee (Proposition~\ref{prop:nonamp}) as a factual
correctness guarantee over tool outputs themselves.

\paragraph{Limitations.}
Our study focuses on image-and-text visual question
answering and leaves several directions for future work.
First, \logotitle{} is evaluated on static multimodal
benchmarks; extending the ledger to long-horizon agents
with persistent memory, video streams, or embodied
interaction is an interesting open problem, and the
time-to-live policy would likely need to be revisited in
such settings. Second, the current grounding cascade
relies on deterministic entity and numeric checks together
with a small predefined alias table, which is effective in
the benchmarks we study but may not cover every form of
paraphrase, coreference, or cross-lingual alias; richer
entity-linking modules could further strengthen the
\textsc{ECC} layer. Third, the Adaptive Dual-Path
Dispatcher uses a rule-based complexity classifier that was
designed for the question distributions we evaluate;
learned dispatcher policies may yield further gains on
more heterogeneous task mixtures. Finally, because
\logotitle{} is training-free and instantiated on frozen
proprietary and open backbones, the absolute numbers will
naturally evolve as backbone MLLMs improve, and we leave a
dedicated study of how ledger-level signals can be turned
into a training signal for MLLMs to future work.

\end{document}